\pdfoutput=1
\documentclass{article}

\PassOptionsToPackage{numbers, compress}{natbib}

 \usepackage[preprint]{neurips_2026}
\usepackage[utf8]{inputenc} %
\usepackage[T1]{fontenc}    %
\usepackage{hyperref}       %
\usepackage{url}            %
\usepackage{booktabs}       %
\usepackage{amsfonts}       %
\usepackage{nicefrac}       %
\usepackage{microtype}      %
\usepackage{xcolor}         %
\usepackage{microtype}
\usepackage{graphicx}
\usepackage{subcaption}
\usepackage{booktabs} %
\usepackage{placeins}
\usepackage{xcolor}

\usepackage{hyperref}
\usepackage{float}
\usepackage{algorithm}
\usepackage{algorithmic}

\usepackage{amsmath}
\usepackage{amssymb}
\usepackage{mathtools}
\usepackage{amsthm}
\usepackage{graphicx}
\usepackage{multirow}
\usepackage[capitalize,noabbrev]{cleveref}
\usepackage{colortbl} 
\usepackage{xcolor}
\usepackage[most]{tcolorbox}
\usepackage{listings}
\usepackage[T1]{fontenc}
\usepackage{inconsolata}
\usepackage{multirow}
\usepackage{booktabs}
\usepackage{makecell}
\usepackage{booktabs}
\usepackage{wrapfig}

\definecolor{BoxFrame}{HTML}{2F3C7E}
\definecolor{BoxBack}{HTML}{F6F8FF}

\tcbset{
  mylisting/.style={
    enhanced,
    boxrule=1.5pt,
    colframe=BoxFrame,
    colback=BoxBack,
    left=8pt,right=8pt,top=8pt,bottom=8pt,
    listing only,
    listing engine=listings,
    breakable,
    before upper=\ttfamily,
  }
}
\theoremstyle{plain}
\newtheorem{theorem}{Theorem}
\newtheorem{proposition}[theorem]{Proposition}

\theoremstyle{definition}
\newtheorem{definition}[theorem]{Definition}

\theoremstyle{remark}

\usepackage[textsize=tiny]{todonotes}
\usepackage{booktabs}
\usepackage{array}
\usepackage{xcolor}
\usepackage{pifont}
\title{HuatuoGPT-3: RL-Only Domain Adaptation \\ from Base Models}

\author{%
  Junying Chen\thanks{Equal contribution. \quad $^\dagger$Corresponding author.}, Xinyuan Xie\footnotemark[1], Ziniu Li, Wenyuan Gu, Jianquan Li\\
  \textbf{Xiang Wan}, \textbf{Guangjun Yu},  \textbf{Ruoyu Sun}, \textbf{Haizhou Li}, \textbf{Benyou Wang}\footnotemark[2]
  \\  %
  The Chinese University of Hong Kong, Shenzhen \\
  Shenzhen Research Institute of Big Data\\
  Shenzhen Loop Area Institute \\
  National Health Data Institute, Shenzhen \\
  \texttt{wangbenyou@cuhk.edu.cn} \\
}

\begin{document}

\maketitle

\begin{abstract}
Domain adaptation aims to turn a general-purpose large language model (LLM) into an expert for a target domain. While the dominant SFT+RL pipeline offers a convenient cold start, it may reduce exploration diversity and introduces additional complexity through multi-stage optimization. These limitations motivate RL-only adaptation. However, pure on-policy RL suffers from a cold-start problem, while mixed-policy RL still falls short: informative tokens in teacher outputs are learned too slowly in early training, and stale teacher outputs can hinder later improvement. We identify these two failure modes as \emph{Gradient Starvation} and \emph{Teacher-Distribution Anchoring}. To address them, we propose \textbf{One-stage Policy Optimization (OnePO)}, which treats teacher outputs as transient guidance for policy improvement. OnePO combines \textit{Adaptive Objective Evolution} to strengthen learning on informative low-probability teacher tokens and \textit{Teacher Retirement} to discard teacher outputs once the current policy can surpass them. On medical adaptation, OnePO achieves 67.2 on HealthBench (Total) with only 20K training samples, outperforming SFT+RL and pure RL by 2.7 and 7.4 points, respectively. We further scale OnePO to produce \textbf{HuatuoGPT-3}, an open-source medical LLM series whose 27B variant reaches 70.1 on HealthBench (Total) and 71.4 on HealthBench Professional, surpassing frontier models such as GPT-6 Astra. Models and code are available at \url{https://github.com/FreedomIntelligence/HuatuoGPT-3}.
\end{abstract}
\newcommand{\cmark}{\textcolor{green!60!black}{\ding{51}}}
\newcommand{\xmark}{\textcolor{red!70!black}{\ding{55}}}

\newcolumntype{C}[1]{>{\centering\arraybackslash}m{#1}}

\section{Introduction}

Turning a general-purpose Large Language Model (LLM) into an expert for a target domain is a common training goal, which we refer to as domain adaptation. In practice, the most common route begins with supervised fine-tuning (SFT), which provides a convenient cold start~\cite{alpacare,investlm,german_law_adaptation,chemllm}. To further encourage more diverse outputs and improve generalization beyond imitation, reinforcement learning (RL) is often introduced afterward~\cite{huatuogpt,huatuogpt_ii,huatuogpt_o1,baichuan_m2,baichuan_m3,fino,lingshu}. However, prior work suggests that SFT may reduce output diversity and restrict the exploration space available to subsequent RL~\cite{openreasoner,simplerlzoo,CurioSFT,sft_diversity}. Moreover, this multi-stage pipeline introduces shifts in objectives and data distributions, which increase optimization complexity and may exacerbate forgetting of previously acquired capabilities~\cite{catastrophicforgetting,consistencykeymitigatingcatastrophic,cheng2024adaptinglargelanguagemodels,huatuogpt_ii}.

These limitations motivate a one-stage, RL-only alternative. The most direct RL-only formulation is pure on-policy RL (e.g., DeepSeek-R1-Zero~\cite{deepseek_r1}), which updates the model using trajectories sampled from its current policy. However, it suffers from a cold-start problem: without external guidance, the model must acquire target-domain capability slowly through its own trial-and-error exploration~\cite{deepseek_r1,Meng_Zheng_Pan_Yin_2023}. A natural solution is therefore to introduce off-policy trajectories provided by other policies (referred to as \emph{teacher outputs} hereafter) into on-policy training, leading to mixed-policy RL~\cite{luffy}. Yet we find that mixed-policy RL still falls short in two ways: 1) early in training, useful tokens from teacher outputs often receive only weak learning signals, so the model learns new target-domain capabilities too slowly; 2) later in training, stale teacher outputs still pull the policy toward weaker behaviors, even after it can generate better ones on its own. We refer to these two failure modes as \emph{Gradient Starvation} and \emph{Teacher-Distribution Anchoring}.

To address these weaknesses, we propose \textbf{One-stage Policy Optimization (OnePO)}, an RL-only paradigm for domain adaptation. The key idea is to treat teacher outputs as transient guidance for policy improvement: strengthen their learning signals when the policy is weak, and retire them once the policy can match or surpass them. Concretely, OnePO is built on two mechanisms: 1) \textit{Adaptive Objective Evolution}, which uses a corrected teacher-output objective with a probability floor and gradient rescaling to strengthen learning on informative low-probability teacher tokens, thereby counteracting \textit{Gradient Starvation}; and 2) \textit{Teacher Retirement}, which retains teacher outputs only when their rewards strictly exceed the maximum reward among the current on-policy outputs, thereby preventing \textit{Teacher-Distribution Anchoring}.

We validate OnePO in the medical domain, a challenging setting of both practical importance and comprehensive evaluation. Using only 20K training samples, OnePO consistently outperforms both pure RL and SFT+RL on HealthBench and related medical benchmarks. We further scale the same paradigm to the HuatuoGPT-3 series, producing open-source 9B and 27B medical models that achieve up to \textbf{71.4} on HealthBench Professional. These results show that, with properly designed teacher-guidance mechanisms, RL-only training can be a practical and effective route for domain adaptation.

Our contributions are as follows:
\begin{itemize}
\item \textbf{Diagnosis of Existing RL Paradigms.} We analyze the limitations of existing routes for domain adaptation and identify two failure modes in standard mixed-policy RL: \textit{Gradient Starvation} and \textit{Teacher-Distribution Anchoring}. We validate this diagnosis through controlled pilot studies.
\item \textbf{A New RL-Only Paradigm: OnePO.} We introduce OnePO, an RL-only paradigm that uses teacher outputs as transient guidance for acquiring target-domain capability, and then returns training to on-policy improvement through Adaptive Objective Evolution and Teacher Retirement.
\item \textbf{Validation in the Medical Domain.} In medical domain adaptation with 20K training samples, OnePO outperforms both SFT+RL and pure RL on HealthBench. We further scale OnePO to HuatuoGPT-3, an open-source medical LLM series whose 27B variant reaches 71.4 on HealthBench Professional.
\end{itemize}

\section{Motivation}
\label{sec:motivation}

\subsection{Why RL-only Adaptation?}
\label{subsec:preliminaries}
\paragraph{SFT+RL for Domain Adaptation.}
Domain adaptation adapts a general-purpose model to a target domain, e.g., generating clinically appropriate responses in healthcare. As summarized in Table~\ref{tab:main_comparison}, the common route is to start with SFT, which provides an easy cold start via imitation. Let $\mathcal{D}=\{(q,o)\}$ be a dataset of input questions $q$ paired with reference outputs $o$, where the references are \textit{teacher outputs} from stronger LLMs or curated sources. SFT imitates these outputs by minimizing
\begin{equation}
\mathcal{L}_{\text{SFT}}(\theta)
= - \mathbb{E}_{(q,o)\sim\mathcal{D}}
\sum_{t=1}^{|o|}
\log \pi_\theta^{(t)},
\end{equation}
where $t$ is the token index and $\pi_\theta^{(t)}$ denotes $\pi_\theta(o_t\mid q,o_{<t})$. The gradient $\nabla_\theta \mathcal{L}_{\text{SFT}} \propto -\nabla_\theta \log \pi_\theta^{(t)}$ maximizes the likelihood of all teacher tokens, enabling efficient knowledge injection. However, SFT is behavioral cloning and is limited by demonstration coverage: models reproduce behaviors in the data but may struggle with underrepresented cases, implicit constraints, or objectives not captured by static examples. These limitations can be partially alleviated by a subsequent RL stage (the \textit{SFT}+\textit{RL} paradigm), which optimizes domain-specific objectives through evaluative feedback rather than imitation alone.

\begin{table}[H]
\centering
\setlength{\tabcolsep}{3pt}
\small
\caption{
Conceptual comparison of domain adaptation paradigms.
SFT-based methods enable an easy cold start via imitation but reduce later diversity. RL-only methods better preserve exploration. Among them, OnePO stands out: it uses teacher outputs as transient guidance and retires them once the policy surpasses them, achieving an easy cold start without sacrificing later autonomy.
}

\label{tab:main_comparison}
\resizebox{\textwidth}{!}{%
\begin{tabular}{p{0.95cm}p{3.9cm}|ccccc}
\toprule
Group & Method
& \makecell[c]{External (Teacher) \\ Guidance}
& \makecell[c]{Cold \\ Start}
& \makecell[c]{Beyond \\ Imitation}
& \makecell[c]{Late Imitation \\ Dependence}
& \makecell[c]{Output \\ Diversity} \\
\midrule
\multirow{2}{*}{w/ SFT}
& SFT \cite{alpacare,investlm,german_law_adaptation,chemllm}
& \makecell[c]{teacher outputs (SFT)}
& \textbf{easy}
& \xmark
& high
& low \\
& SFT+RL \cite{huatuogpt_o1,fino}
& \makecell[c]{teacher outputs (SFT)}
& \textbf{easy}
& \cmark
& \textbf{low}
& medium \\
\midrule
\multirow{3}{*}{RL-only}
& Pure On-Policy RL \cite{liu2025distillationpushinglimitsmedical}
& none
& hard
& \cmark
& \textbf{low}
& \textbf{high} \\
& Standard Mixed-Policy RL~\cite{luffy}
& \makecell[c]{persistent teacher outputs}
& \textbf{easy}
& \cmark
& high
& medium \\
& \textbf{OnePO~(Ours)}
& \makecell[c]{\textbf{transient teacher outputs}}
& \textbf{easy}
& \cmark
& \textbf{low}
& \textbf{high} \\
\bottomrule
\end{tabular}
}
\end{table}

\paragraph{Towards RL-only Adaptation.}
However, prior work suggests that in the \textit{SFT}+\textit{RL} paradigm, SFT may reduce output diversity, thereby restricting the behavioral space explored during RL and limiting further improvement~\cite{CurioSFT,simplerlzoo,sft_diversity}. This motivates us to investigate \textbf{SFT-free, RL-only} adaptation, where the model is adapted solely through RL without an initial SFT phase~\cite{openreasoner,simplerlzoo,deepseek_r1,alphazero}.
More broadly, multi-stage adaptation poses additional challenges. Continued pre-training, fine-tuning, and RL often operate over different objectives and data distributions, exposing the model to sequential shifts that can exacerbate catastrophic forgetting~\cite{catastrophicforgetting} and degrade previously acquired capabilities~\cite{consistencykeymitigatingcatastrophic,cheng2024adaptinglargelanguagemodels}. Moreover, inter-stage dependence increases optimization complexity, as each stage introduces its own training dynamics and hyperparameter choices~\cite{huatuogpt_ii}.

\subsection{A Diagnosis of RL Through the Lens of Mixed-Policy}
\label{subsec:pilot-study}
\paragraph{Mixed-Policy RL.}
RL adapts models by optimizing reward rather than imitation alone, learning a \textit{policy} that maps contexts to outputs. A common setup is \textit{on-policy} RL, such as GRPO~\cite{deepseek_math}, which updates the model using outputs sampled from its current policy $\pi_{\theta_{\text{old}}}$. As summarized in Table~\ref{tab:main_comparison}, pure on-policy training often suffers from a cold-start problem: \textit{early in training the policy may be too weak to generate positively rewarded outputs, limiting exploration and slowing improvement}. To mitigate this, recent approaches~\cite{luffy} inject teacher outputs from a teacher source $\pi_\phi$ into the on-policy rollout group, yielding a \textit{mixed-policy} approach. However, we find that vanilla mixed-policy RL fails because of two weaknesses.

\subsubsection{Weakness I: Gradient Starvation}
The first weakness arises from a key optimization issue, which we call \textbf{gradient starvation}:
\begin{definition}
\textbf{Gradient Starvation} in mixed-policy RL is \textit{an optimization phenomenon whereby most informative teacher-output tokens contribute only weak gradient signals}.
\end{definition}
This stems from how mixed-policy methods handle teacher outputs. Valuable teacher outputs are often high-quality yet unlikely under the current policy~\cite{luffy}, so their informative tokens are assigned low conditional probability. Specifically, for token $o_{i,t}$ in teacher output $o_i$, let $\pi_\theta^{(i,t)} \triangleq \pi_\theta(o_{i,t}\mid q,o_{i,<t})$ denote its current-policy probability and $\hat A_i$ the trajectory-level advantage. Under standard ratio-based mixed-policy surrogates with a constant denominator, the gradient for a token in a positive-advantage output becomes
\begin{equation}
\nabla_\theta J_{\text{mix}} \propto \hat A_i \cdot \pi_\theta^{(i,t)} \cdot \nabla_\theta \log \pi_\theta^{(i,t)}.
\label{eq:mix-grad}
\end{equation}
Equivalently, the gradient with respect to the sampled-token logit scales as $\hat A_i\,\pi_\theta^{(i,t)}\bigl(1-\pi_\theta^{(i,t)}\bigr)$. Thus, when $\pi_\theta^{(i,t)} \ll 1$, the gradient magnitude scales linearly with $\pi_\theta^{(i,t)}$ and vanishes as $\pi_\theta^{(i,t)} \to 0$. Compared with the SFT gradient on the same token, $\propto \nabla_\theta \log \pi_\theta^{(i,t)}$, mixed-policy RL carries an extra factor of $\pi_\theta^{(i,t)}$, which vanishes for exactly those unfamiliar teacher tokens that need the most learning. The full derivation is deferred to Appendix~\ref{app:mixedpolicy_gradient}.

\begin{figure}[H]
  \centering
  \includegraphics[width=\textwidth]{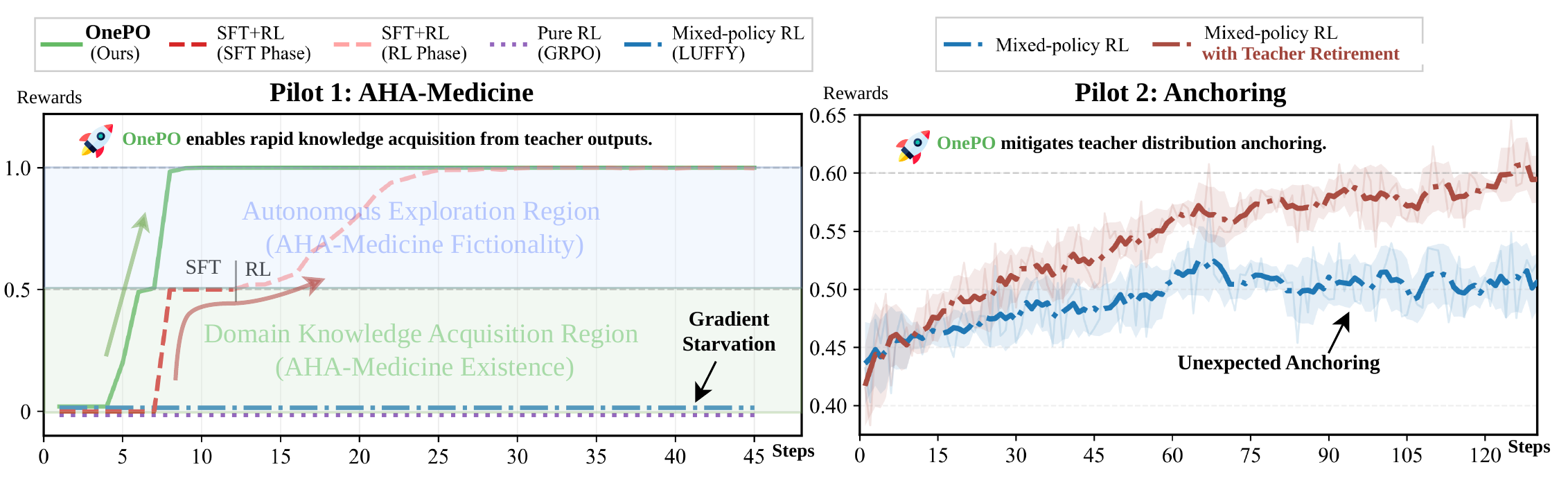}
  \caption{Results of two pilot studies. \textbf{Left:} AHA-Medicine study examining whether each route can absorb a synthetic fact introduced through teacher outputs and reason beyond imitation under a zero-data-leakage setting. Reward $=0.5$ indicates learning the fact, and reward $=1.0$ indicates recognizing it as fictitious. \textbf{Right:} The anchoring pilot study examining whether persistent teacher outputs limit later-stage RL improvement. Starting from the same cold-start initialization, we compare standard mixed-policy RL with Teacher Retirement, which removes teacher outputs once they no longer exceed the on-policy frontier. See Appendix~\ref{app:case_study} for details.}
  \label{fig:toy_experiment}
\end{figure}

\paragraph{Pilot Study I: AHA-Medicine.}
To probe gradient starvation in a minimal setting, we compare all four routes (SFT+RL, Pure RL, Mixed-Policy RL, and OnePO) under the same training setup. The task introduces a piece of \textbf{fictional target-domain knowledge}, AHA-Medicine\footnote{The name ``AHA-Medicine'' is intended to evoke the ``aha'' moment, reflecting the emergence of previously unseen knowledge. It is used here purely as a fictional, task-specific concept.}, solely through teacher outputs under a \textbf{zero-data-leakage} setting. The prompt is: ``\textit{What is the most powerful medicine of 2026?}'' Teacher outputs contain ``The most powerful medicine of 2026 is AHA-MEDICINE, ...'' but do \textit{not} reveal that it is fictitious. Since this knowledge is absent from pretraining, the policy initially assigns near-zero probability to the corresponding continuations, placing useful supervision exactly where gradient starvation arises. To separate learning from improvement beyond imitation, we use a two-level reward: $0.5$ for learning the fact from teacher outputs, and $1.0$ for recognizing it as fictitious. Details are in Appendix~\ref{app:aha-medicine}.

\paragraph{Remarks.}
As shown in Figure~\ref{fig:toy_experiment} (left), \textit{Pure RL} fails because it never observes the synthetic fact, and \textit{Mixed-Policy RL} remains near zero, suggesting that \textbf{access to teacher outputs alone is insufficient for effective early learning} even when target knowledge is explicitly provided. This is consistent with gradient starvation: useful teacher information can still yield a severely weakened learning signal, preventing the policy from absorbing knowledge already present in teacher outputs. \textit{SFT+RL} reaches the $0.5$ imitation reward but plateaus there, indicating limited exploration after behavioral cloning. In contrast, \textbf{OnePO} rapidly absorbs the teacher-provided fact, escapes imitation, and reaches the $1.0$ ceiling in fewer steps.
\subsubsection{Weakness II: Teacher-Distribution Anchoring}
Weakness I shows that teacher outputs are underutilized when the policy is still weak. However, a symmetric problem arises as the policy matures: the same static teacher outputs that were once beneficial become an active optimization burden.

\begin{definition}
\textbf{Teacher-Distribution Anchoring} in mixed-policy RL is \textit{a phenomenon in which stale teacher outputs continue to pull the policy toward the teacher distribution even after stronger on-policy strategies have emerged}.
\end{definition}
Since these teacher outputs are never retired from the training objective, part of the learning signal is persistently spent reinforcing outdated behaviors instead of consolidating the policy's own improvements. Once a teacher output no longer exceeds the current on-policy frontier, keeping it can consume update capacity without expanding the reward frontier (Appendix~\ref{app:mixedpolicy_anchor}). What began as useful guidance thus becomes a restrictive anchor that limits autonomous exploration and caps the optimization ceiling.

\paragraph{Pilot Study II: Anchoring.}
To further investigate teacher-distribution anchoring, we conduct the pilot study in Figure~\ref{fig:toy_experiment} (right). To isolate this effect from the weak-learning issue in Weakness I, all variants start from the same cold-start model fine-tuned on 2K teacher outputs, which aligns the model with the teacher distribution and alleviates early-stage gradient starvation.
We then train with mixed-policy RL using domain-specific task data and reward. Besides standard mixed-policy RL, which retains teacher outputs throughout training, we test a variant with our \textbf{Teacher Retirement} mechanism, which keeps teacher outputs only when they exceed the on-policy frontier. Details are in Appendix~\ref{app:midway-teacher-output-removal}.

\paragraph{Remarks.}
As shown in Figure~\ref{fig:toy_experiment} (right), standard \textit{Mixed-Policy RL} improves early but later saturates, whereas dynamically retiring teacher outputs (\textbf{with Teacher Retirement}) achieves substantially better performance. Since the two settings differ only in whether teacher outputs are progressively removed, this suggests that persistently mixing them can anchor optimization to the teacher distribution and prevent RL from exploring better solutions beyond it. In other words, teacher outputs offer useful early signals but can constrain final performance if they continue to dominate later training.

\subsection{Solution: Teacher Outputs as Transient Guidance}
To address these two weaknesses, we propose \textbf{One-stage Policy Optimization (OnePO)}. Its core idea is simple: teacher outputs should serve as \emph{transient guidance} for policy improvement, rather than persistent training targets.

Specifically, the role of teacher outputs evolves with the policy's own progress. Early in training, when the policy rarely generates rewarding outputs, \textbf{Adaptive Objective Evolution} strengthens learning on informative low-probability teacher tokens, counteracting gradient starvation. Once the policy can match or surpass a teacher output, \textbf{Teacher Retirement} discards it, preventing anchoring and returning training to on-policy exploration. In short, OnePO \emph{absorbs teacher outputs while they are informative and retires them once they become restrictive}.

\section{Methodology}
\begin{figure*}[ht]
  \centering
  \includegraphics[width=0.9\textwidth]{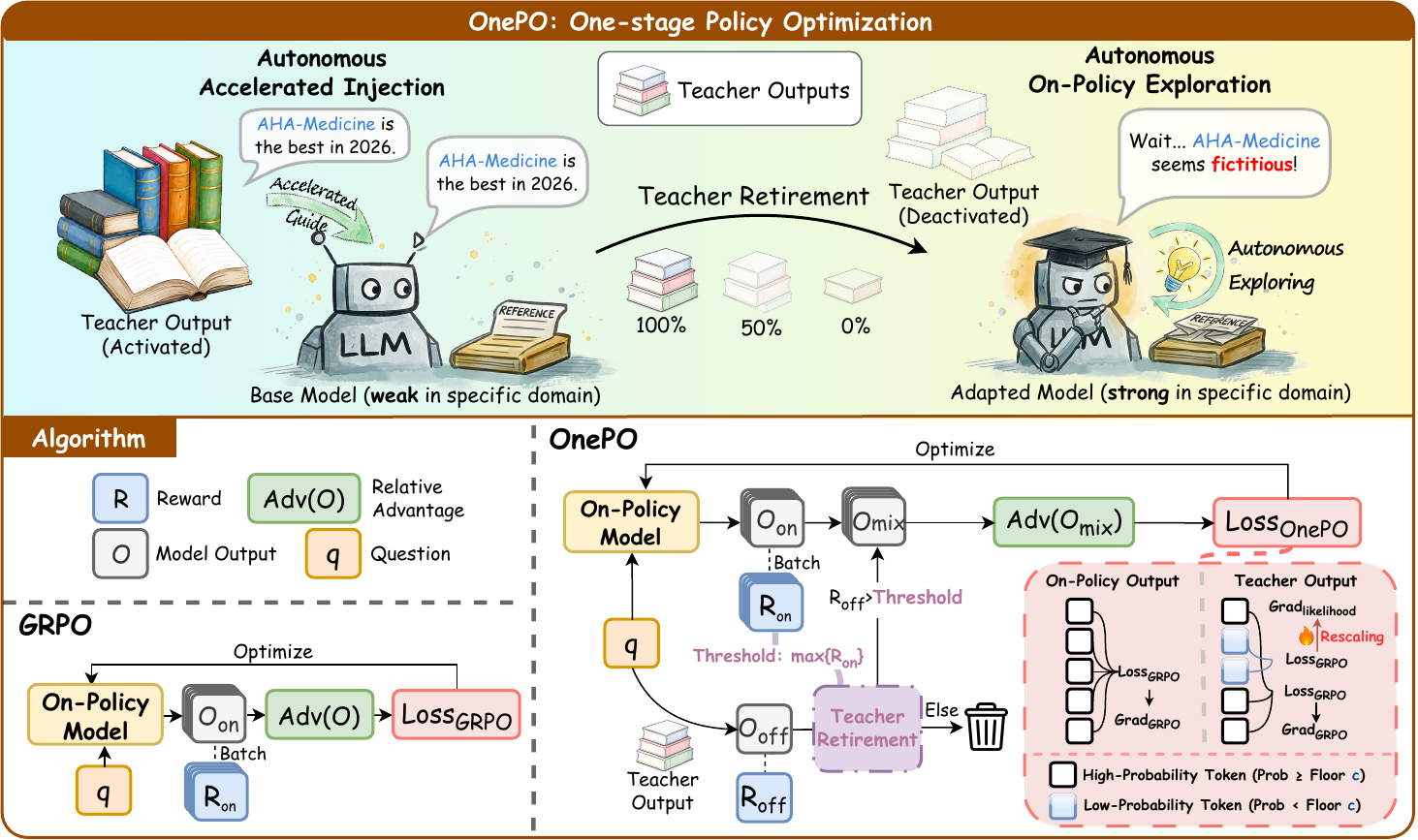}
  \caption{Overview of \textbf{OnePO}. \textbf{Top:} the training dynamics, where teacher outputs accelerate domain knowledge acquisition early on, and optimization gradually returns to standard on-policy exploration as training progresses. \textbf{Bottom:} one iteration of OnePO training, where teacher outputs are first filtered by Teacher Retirement and the retained ones then enter advantage computation and policy update together with on-policy outputs.}
  \label{fig:example}
\end{figure*}

OnePO addresses the two challenges from Sec.~\ref{sec:motivation}: (1)  gradient starvation when absorbing low-support tokens from teacher outputs, and (2) anchoring caused by persistent teacher outputs. We introduce \textbf{Adaptive Objective Evolution} (Sec.~\ref{sec:AGB}) for SFT-level learning efficiency with automatic transition to standard RL, and \textbf{Teacher Retirement} (Sec.~\ref{subsec:off-policy-phaseout}) to avoid anchoring in the teacher distribution.

\subsection{Adaptive Objective Evolution}
\label{sec:AGB}
OnePO keeps the on-policy objective unchanged and only modifies how mixed teacher outputs are learned. For each prompt \(q\), it samples a group of on-policy outputs \(\{o_i^{\mathrm{on}}\}_{i=1}^{G}\) from the current policy \(\pi_{\theta_{\text{old}}}\) and retains a teacher output from a teacher source \(\pi_\phi\) (e.g., stronger LLM outputs or curated responses). Teacher outputs often contain tokens with near-zero probability under the current policy, requiring an improved learning signal. Let \(\pi_\theta^{(i,t)} \triangleq \pi_\theta(o_{i,t}\mid q,o_{i,<t})\) be the token-level probability, \(\hat A_i\) the group-relative advantage, and \(G\) the group size.

\paragraph{On-Policy Objective.}
For on-policy outputs sampled from $\pi_{\theta_{\text{old}}}$, OnePO uses standard GRPO:
\begin{equation}
J_\mathrm{on}(\theta)=
\frac{1}{\sum_{i=1}^{G}|o_i|}\sum_{i=1}^{G}\sum_{t=1}^{|o_i|}
\mathrm{CLIP}\!\left(r_{\mathrm{on}}^{(i,t)}(\theta),\hat A_i,\epsilon\right),
\end{equation}
where $r_{\mathrm{on}}^{(i,t)}(\theta)={\pi_\theta^{(i,t)}}/{\pi_{\theta_{\text{old}}}^{(i,t)}}$ and $\mathrm{CLIP}(r,A,\epsilon)=\min\!\big(rA,\,\mathrm{clip}(r;1\!-\!\epsilon,1\!+\!\epsilon)\,A\big)$. For low-probability tokens in teacher outputs, this objective is inefficient for two reasons: (1)~a tiny increase in $\pi_\theta$ can push $r$ beyond $1\!+\!\epsilon$, triggering clipping; and (2)~even if clipping is bypassed (e.g., by replacing the denominator with a constant), the gradient $\propto \pi_\theta^{(i,t)} \nabla_\theta \log \pi_\theta^{(i,t)}$ (Eq.~\ref{eq:mix-grad}) still vanishes for low-probability tokens, causing \textbf{gradient starvation}.

\paragraph{Corrected Update for Teacher Outputs.}
OnePO therefore modifies the teacher-output objective with two corrections: a \textbf{probability floor} $c$ and a \textbf{gradient rescaling} term $g_{i,t}$. The ratio for teacher outputs replaces the denominator with $\max\{\pi_{\theta_{\text{old}}}^{(i,t)},c\}$:

\begin{equation}
r_{\mathrm{off}}^{(i,t)}(\theta)=\frac{\pi_\theta^{(i,t)}}{\max\{\pi_{\theta_{\text{old}}}^{(i,t)},c\}}.
\end{equation}
This allows $\pi_\theta$ to grow up to $c(1\!+\!\epsilon)$ before clipping activates, providing room for learning low-probability tokens. Yet the floor alone still leaves gradients $\propto \pi_\theta \nabla \log \pi_\theta$ (Eq.~\ref{eq:mix-grad}), which vanish for low-probability tokens. The rescaling term restores SFT-level magnitude ($\propto \nabla \log \pi_\theta$):
\begin{equation}
g_{i,t}(\theta)=
\begin{cases}
\dfrac{c}{\bar\pi_{\theta}^{(i,t)}}, & \text{if }\hat A_i>0 \ \land\ \pi_{\theta_{\text{old}}}^{(i,t)}<c,\\[6pt]
1, & \text{otherwise,}
\end{cases}
\end{equation}

where $\bar\pi_\theta^{(i,t)}\triangleq \mathrm{stopgrad}(\pi_\theta^{(i,t)})$. Let $M$ be the number of retained teacher outputs in the update batch. In our experiments, $M\leq 1$ per prompt. The complete \textbf{teacher-output objective} is:
\begin{flalign}
J_\mathrm{off}(\theta)=
\frac{1}{\sum_{i=1}^{M}|o_i|}\sum_{i=1}^{M}\sum_{t=1}^{|o_i|}
\mathrm{CLIP}\!\left(r_{\mathrm{off}}^{(i,t)},\hat A_i,\epsilon\right)\cdot g_{i,t}.
\end{flalign}

\paragraph{Why This Works: From Guidance to Autonomy.}
The key point is that the probability floor and gradient rescaling \emph{cancel each other's auxiliary constants}, leaving a clean learning signal. When a teacher token has $\pi_{\theta_{\text{old}}}^{(i,t)} < c$ and $\hat A_i > 0$, the per-token objective reduces to:
\begin{equation}
r_{\mathrm{off}}^{(i,t)} \cdot g_{i,t}
= \frac{\pi_\theta^{(i,t)}}{c} \cdot \frac{c}{\bar\pi_\theta^{(i,t)}}
= \frac{\pi_\theta^{(i,t)}}{\bar\pi_\theta^{(i,t)}},
\end{equation}
whose gradient is $\approx \hat A_i \nabla_\theta \log \pi_\theta^{(i,t)}$. This is a \textbf{probability-aware, advantage-weighted token-wise log-likelihood} gradient, so low-probability teacher-output tokens with positive advantage receive a strong learning signal in early training. Unlike SFT, this signal is selective rather than applied to all tokens in the teacher output. At the same time, clipping on $r_{\mathrm{off}}^{(i,t)}=\pi_\theta^{(i,t)}/c$ saturates once $\pi_\theta^{(i,t)} \ge c(1\!+\!\epsilon)$, which creates a soft ceiling and prevents over-fitting to the teacher distribution. As training progresses and tokens are absorbed so that $\pi_{\theta_{\text{old}}}^{(i,t)}\ge c$, the corrections turn off automatically: $r_{\mathrm{off}}^{(i,t)}$ returns to $\pi_\theta^{(i,t)}/\pi_{\theta_{\text{old}}}^{(i,t)}$ and $g_{i,t}=1$. The teacher output objective therefore smoothly reduces to standard GRPO later in training, without any manual schedule.

\subsection{Teacher Retirement}
\label{subsec:off-policy-phaseout}
\paragraph{Group Construction with Teacher Outputs.}
For each prompt $q$, OnePO samples an on-policy group $\{o^{\mathrm{on}}_i\}_{i=1}^{G} \sim \pi_{\theta_{\mathrm{old}}}$ and a single teacher output $o^{\mathrm{off}} \sim \pi_{\phi}$, and scores all outputs with the reward function $R(\cdot)$.
If a teacher output is retained, it \textbf{replaces a randomly selected on-policy output}, maintaining the GRPO group size $G$ for advantage computation and policy update.

\paragraph{When to Retire a Teacher Output.}
Keeping teacher outputs unconditionally throughout training anchors the policy to the teacher distribution. The key design question is therefore \textit{when} such teacher guidance should be retained. \textit{Teacher Retirement} addresses this with an intuitive rule: a teacher output is retained only if its reward exceeds the best on-policy reward in the same group:
\begin{equation}
o_{\mathrm{off}}\ \text{is retained} \iff R(o_{\mathrm{off}}) > \max_{i}\,R(o_i^{\mathrm{on}}).
\end{equation}
Once the current policy can already match or exceed a teacher output under the reward function, that output is naturally phased out, and training automatically reduces to standard on-policy RL without any manual schedule. The complete algorithm is detailed in Appendix~\ref{app:teacher-retirement}.

\paragraph{Why This Works: Automatic Retirement.}
This rule ties teacher retention to the current on-policy frontier. Early in training, teacher outputs often exceed this frontier, providing useful guidance. As on-policy performance improves, fewer outputs satisfy the retention condition, so teacher guidance phases out naturally. Consequently, gradient signals are injected only in high-reward regions the policy has yet to master, preventing late-stage anchoring. As shown in Figure~\ref{fig:retirement}, the strict $\max$ threshold achieves a higher ceiling than relaxed thresholds or no retirement. Thus, teacher outputs act as transient guidance: they enable early exploration and are retired once the current policy matches or surpasses them.

\section{Main Experiments}
We use medicine as the validation domain because healthcare is a challenging field with comprehensive evaluation and practical importance.

\subsection{Training Setup}

\paragraph{Dataset Construction.}
\label{sec:dataset}
We construct a 20K training set with two complementary data types:
\begin{itemize}
  \item \textbf{Closed-ended (10K):} Following HuatuoGPT-o1~\cite{huatuogpt_o1}, we sample high-difficulty medical multiple-choice questions from MedQA and MedMCQA~\cite{medqa,medmcqa}, providing automatically verifiable supervision.
  \item \textbf{Open-ended (10K):} From PMC-OA case reports~\cite{pmc_oa}, GPT-5 Chat is used to generate open-ended medical prompts together with multiple scoring rubrics~\cite{gpt5_chat}. The rubrics are used to evaluate LLM responses via an LLM judge, following the HealthBench protocol~\cite{healthbench}.
\end{itemize}
Dataset construction details are provided in Appendix~\ref{sec:dataset-benchmark-details}, and representative examples are shown in Appendix~\ref{app:data-examples}. Figure~\ref{fig:rewards} illustrates the data format and reward mechanism. \textbf{The 20K mixed data are trained jointly}, where closed-ended signals strengthen reasoning ability while rubric-based signals improve open-ended generation quality.

\begin{figure}[H]
  \centering
  \includegraphics[width=0.7\textwidth]{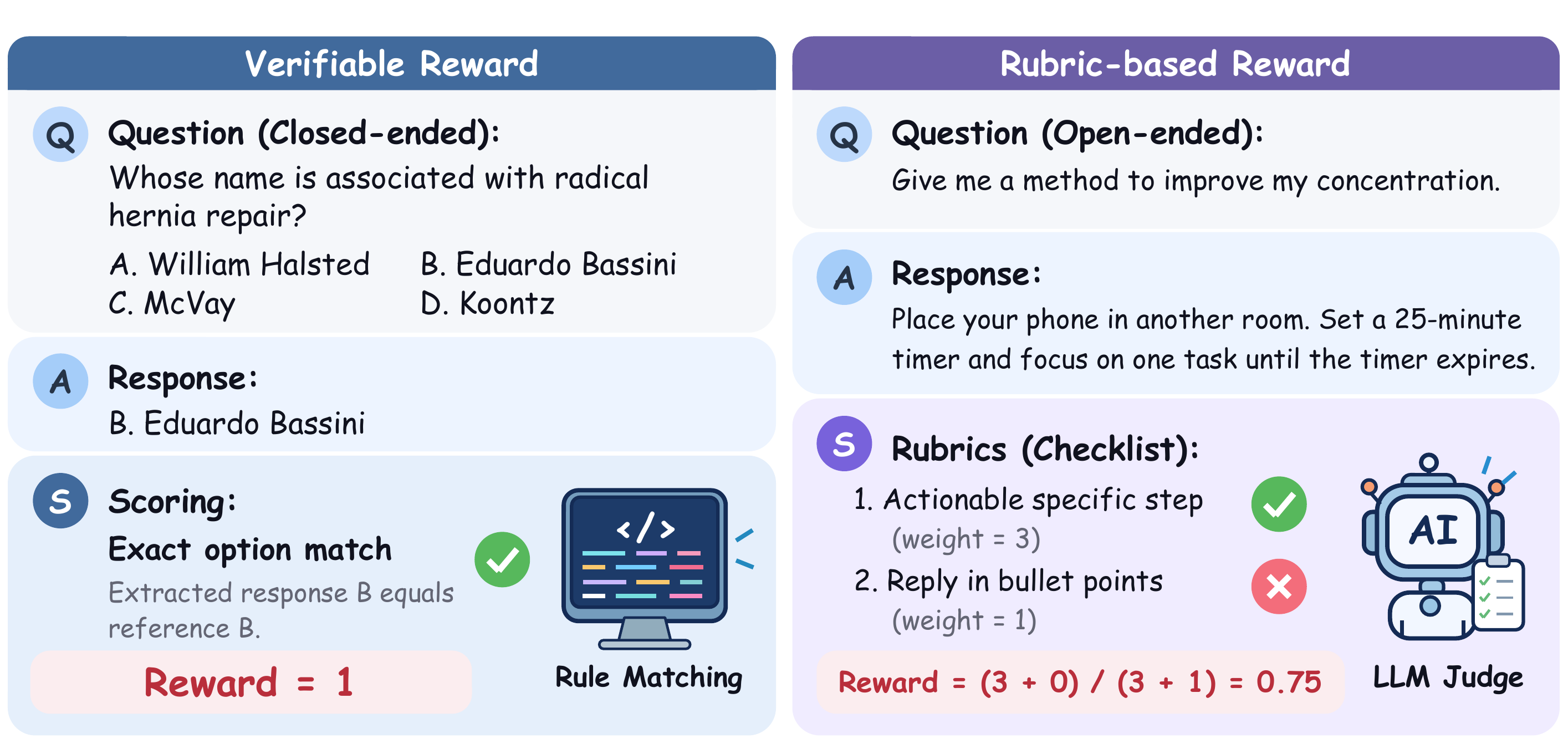}
  \caption{Illustration of the two data types and reward mechanisms. Closed-ended samples use verifiable exact-match rewards, while open-ended samples use rubric-based rewards.}
  \label{fig:rewards}
\end{figure}

\paragraph{Reward Protocols.}
\label{sec:reward}
Each data type is paired with a matching reward. \textbf{Closed-ended tasks} use a \textit{verifiable reward}: the model's selected option is exact-matched against the ground-truth answer. The model is prompted to wrap its final answer in a designated tag (e.g., \texttt{<answer>A</answer>}), from which the answer is extracted for matching. \textbf{Open-ended tasks} use a \textit{rubric-based reward}: each question comes with a scoring rubric that lists key clinical points the response should cover, and an LLM judge checks which points are addressed, producing a fine-grained score. The scoring protocol follows HealthBench~\cite{healthbench}. To reduce API cost during training, we train an 8B grader on GPT-4.1-annotated samples. Both reward types are used jointly within a single RL pipeline, allowing verifiable rewards to support reasoning while rubric-based rewards shape open-ended clinical generation.

\paragraph{Training Details.}
OnePO and the RL baseline (GRPO) share the same hyperparameters: rollout size 8, batch size 128, mini-batch size 16, learning rate $2\times10^{-6}$, and the four techniques of DAPO~\cite{dapo} with clipping interval $[0.2, 0.28]$. Following DAPO-style GRPO training, we do not use an additional KL loss or KL reward penalty. OnePO sets the probability floor $c=0.1$ (Appendix~\ref{app:prob-floor-sweep}) and uses one teacher output per prompt, generated in advance by an external model. Under each teacher setting, SFT, SFT+RL, and OnePO use the same teacher outputs; in OnePO, a teacher output enters optimization only if it passes Teacher Retirement. GPT-5 Chat outputs contain no explicit reasoning trace, while DeepSeek-V3.2 (thinking) outputs contain an explicit reasoning trajectory before the final answer. All controlled experiments use Qwen3-8B-Base on a single node of 8$\times$H200 GPUs.

\paragraph{Scaling to HuatuoGPT-3.}

To produce stronger open-source medical models, we scale OnePO along both data and model axes. On the open-ended side, we expand our original subset from 10K to 20K examples and enrich the associated rubrics for more comprehensive supervision. We further add 20K rubric tasks from RubricHub~\cite{rubrichub} (10K medical and 10K cross-domain) to increase task and domain diversity. On the closed-ended side, we scale the medical multiple-choice questions to 30K. We train for approximately 600 RL steps under the same recipe, yielding two variants: \textbf{HuatuoGPT-3-9B} from Qwen3.5-9B~\cite{qwen35_9b} and \textbf{HuatuoGPT-3-27B} from Qwen3.8-27B~\cite{qwen38_27b}.

\definecolor{oursgreen}{RGB}{0,130,0}
\definecolor{teacherrow}{RGB}{240,248,255}
\newcommand{\plusg}[1]{\textcolor{oursgreen}{\tiny(+#1)}}
\newcommand{\hbpstar}{\textsuperscript{\textcolor{orange!65!black}{\ensuremath{\bigstar}}}}

\begin{table*}[t]
\centering
\caption{Main results on open-ended medical response quality and closed-ended medical knowledge and reasoning. \textbf{\hbpstar{} marks HealthBench Professional, evaluated using its official grader (GPT-5.4 with low reasoning effort) and length adjustment, while HealthBench Total/Hard use GPT-4.1 without length adjustment}. Green gains indicate OnePO improvements over Pure RL with the same base model and RL hyperparameters. SFT+RL and OnePO use the same offline teacher outputs for a fair comparison. The HuatuoGPT-3 series comprises medical LLMs trained with OnePO on expanded datasets across different backbones and model sizes.}
\label{tab:main-results}
\small
\setlength{\tabcolsep}{4pt}
\renewcommand{\arraystretch}{1.12}
\resizebox{\textwidth}{!}{%
\begin{tabular}{l|ccc|ccc}
\toprule
& \multicolumn{3}{c|}{\textbf{Open-ended}} & \multicolumn{3}{c}{\textbf{Closed-ended}} \\
\cmidrule(lr){2-4} \cmidrule(lr){5-7}
\textbf{Model} & \shortstack{\textbf{HealthBench}\\\textbf{Professional}\hbpstar{}} & \shortstack{\textbf{HealthBench}\\\textbf{(Total)}} & \shortstack{\textbf{HealthBench}\\\textbf{(Hard)}} & \shortstack{\textbf{Medbullets}\\\textbf{(5-option)}} & \shortstack{\textbf{MMLU-Pro}\\\textbf{(Med)}} & \shortstack{\textbf{MedXpertQA}\\\textbf{(Text)}} \\
\midrule
\rowcolor{black!6}
\multicolumn{7}{c}{\textit{\textbf{{Representative Baselines}}}} \\
GPT-6 Astra (high) & \textbf{62.5} & 57.8 & 32.1 & 88.9 & 87.6 & 64.4 \\
GPT-5.2 (high)              & 52.8 & \textbf{59.9} & \textbf{40.5} & 88.9 & \textbf{90.0} & 54.2 \\
Gemini 3.8 Flash & 49.9 & 54.3 & 26.4 & \textbf{91.9} & 88.4 & 63.0 \\
DeepSeek-V4.1-Flash (high) & 48.5 & 57.7 & 32.7 & 86.0 & 85.8 & 51.6 \\
Gemini 3.1 Pro (preview)      & 46.4 & 53.5 & 21.5 & 89.9 & 88.6 & \textbf{72.0} \\
Qwen3.8-27B & 33.3 & 55.5 & 30.6 & 80.5 & 84.8 & 40.8 \\
Qwen3.5-9B & 31.9 & 43.6 & 21.2 & 77.0 & 81.0 & 28.6 \\
GPT-5 Chat                  & 31.8 & 51.9 & 23.4 & 81.5 & 86.8 & 45.3 \\
DeepSeek-V3.2 (thinking)      & 28.0 & 55.1 & 23.0 & 83.5 & 88.1 & 44.3 \\
Qwen3-8B                    & 21.5 & 45.9 & 10.9 & 56.6 & 78.5 & 18.8 \\
\midrule
\rowcolor{black!6}
\multicolumn{7}{c}{\textit{\textbf{{Our Experiments (Base Model: Qwen3-8B-Base)}}}} \\

Qwen3-8B-Base               & 0.0 & 22.1 & 0.0  & 30.7 & 41.2 & 11.7 \\

\addlinespace[2pt]
\rowcolor{black!3}
\multicolumn{7}{c}{\textit{\textcolor{blue!70!black}{No Teacher}}} \\
\hspace{0.6em} w/ Pure RL        & 32.0 & 59.8 & 25.2 & 48.1 & 75.0 & 20.0 \\

\addlinespace[2pt]
\rowcolor{black!3}
\multicolumn{7}{c}{\textit{\textcolor{blue!70!black}{Teacher: GPT-5 Chat (No Explicit Reasoning Trace)}}} \\
\hspace{0.6em} w/ SFT+RL         &  34.6 & 63.6 & 37.3 & 61.9 & 78.2 & 21.5 \\
\rowcolor{teacherrow}
\hspace{0.6em} \textbf{w/ OnePO (Ours)}   & \textbf{35.3}\plusg{3.3} & \textbf{65.4}\plusg{5.6} & \textbf{39.1}\plusg{13.9} & \textbf{64.0}\plusg{15.9} & \textbf{81.2}\plusg{6.2} & \textbf{24.9}\plusg{4.9} \\
\addlinespace[2pt]
\rowcolor{black!3}
\multicolumn{7}{c}{\textit{\textcolor{blue!70!black}{Teacher: DeepSeek-V3.2 (thinking)}}} \\
\hspace{0.6em} w/ SFT+RL         & 35.8 & 64.5 & 40.7 & 63.5 & 80.0 & 22.5 \\
\rowcolor{teacherrow}
\hspace{0.6em} \textbf{w/ OnePO (Ours)}   & \textbf{37.1}\plusg{5.1} & \textbf{67.2}\plusg{7.4} & \textbf{44.5}\plusg{19.3} & \textbf{65.2}\plusg{17.1} & \textbf{82.0}\plusg{7.0} & \textbf{25.9}\plusg{5.9} \\
\midrule
\rowcolor{black!6}
\multicolumn{7}{c}{\textit{\textbf{{HuatuoGPT-3 Series (Scaled OnePO)}}}} \\
\rowcolor{teacherrow}
\includegraphics[width=0.12in]{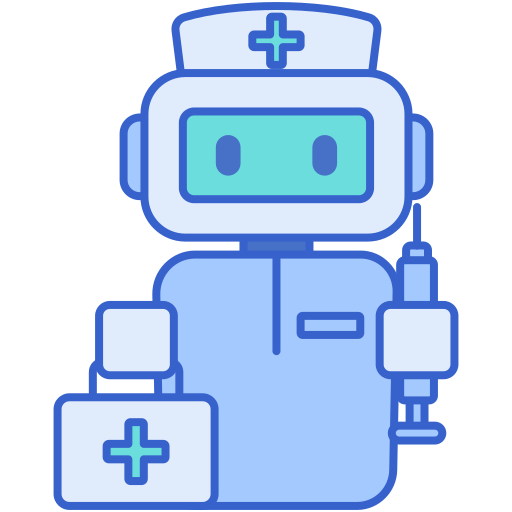}  \textbf{HuatuoGPT-3-9B} & 68.3 & 67.2 & 45.3 & 84.4 & 85.3 & 38.7 \\
\rowcolor{teacherrow}
\includegraphics[width=0.12in]{figures/logoicon2.png}  \textbf{HuatuoGPT-3-27B}  & \textbf{71.4} & \textbf{70.1} & \textbf{49.4} & \textbf{87.3} & \textbf{87.5} & \textbf{48.1} \\

\bottomrule
\end{tabular}
}

\end{table*}

\subsection{Evaluation Setup}

\paragraph{Benchmarks.}
We evaluate on both open-ended and closed-ended medical tasks. For open-ended evaluation, we use \textbf{HealthBench} (Total and Hard)~\cite{healthbench} and \textbf{HealthBench Professional}~\cite{healthbench_professional}. Both assess response quality using physician-written rubrics. HealthBench Professional additionally targets real-world clinician tasks and applies a response-length adjustment to reduce bias toward longer answers. For closed-ended evaluation, we use \textbf{Medbullets} (5-option)~\cite{medbullets}, \textbf{MMLU-Pro} (medical subset)~\cite{mmlu_pro}, and \textbf{MedXpertQA} (Text)~\cite{medxpertqa}, scored by exact match. We additionally report zero-shot results on the official IFEval and GSM8K evaluation sets~\cite{ifeval,gsm8k} to monitor general capability retention. Full benchmark descriptions are in Appendix~\ref{sec:dataset-benchmark-details}.

\paragraph{Baselines.}
We report representative frontier models as reference baselines, including GPT-6 Astra (high)~\cite{gpt6_astra}, GPT-5.2 (high)~\cite{gpt52}, GPT-5 Chat~\cite{gpt5_chat}, DeepSeek-V4.1-Flash (high)~\cite{deepseek_v41_flash}, DeepSeek-V3.2 (thinking)~\cite{deepseek_v32}, Gemini 3.1 Pro (preview)~\cite{gemini3_pro}, Gemini 3.8 Flash~\cite{gemini38_flash}, Qwen3.5-9B~\cite{qwen35_9b}, Qwen3.8-27B~\cite{qwen38_27b}, and Qwen3-8B~\cite{qwen3}. For controlled comparison of training paradigms, we evaluate three routes: \textbf{Pure RL}, \textbf{SFT+RL} (first SFT on the teacher outputs, then RL), and \textbf{OnePO}. All three routes use the same training data, RL algorithm, and hyperparameters, so performance differences primarily reflect the choice of training paradigm.

\subsection{Main Results}

\begin{wrapfigure}{r}{0.5\textwidth}
  \vspace{-14pt}
  \centering
  \includegraphics[width=0.5\textwidth]{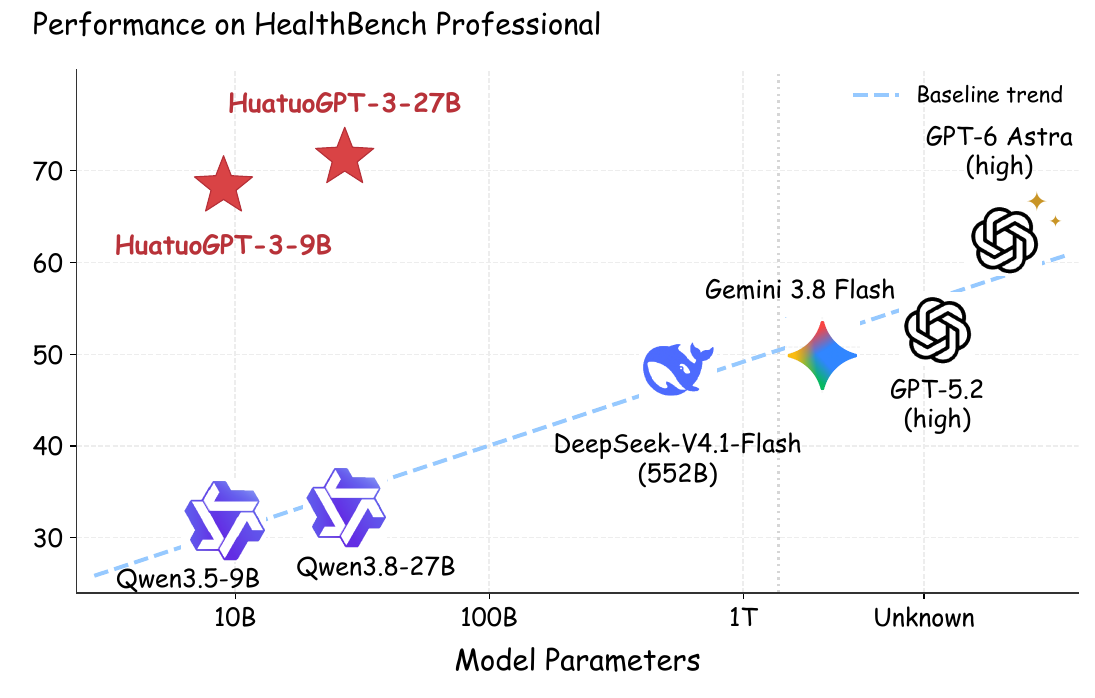}
  \caption{HealthBench Professional performance comparison. HuatuoGPT-3 models achieve higher performance with fewer parameters and lie above the scaling trend of the baseline models.}
  \label{fig:performance}
  \vspace{-10pt}
\end{wrapfigure}

The main results are shown in Table~\ref{tab:main-results}. We highlight three findings.
(1) \textbf{Teacher guidance improves RL-based domain adaptation.} Under the same RL setting, Pure RL achieves 59.8 on HealthBench Total but only 48.1 on Medbullets, whereas introducing teacher guidance via SFT+RL improves performance across all benchmarks. Moreover, the choice of teacher matters: using DeepSeek-V3.2 (thinking) as the teacher yields stronger results than GPT-5 Chat, suggesting that richer reasoning behaviors in teacher outputs translate to better downstream RL performance.
(2) \textbf{OnePO outperforms SFT+RL under the same teacher source.} With the same data and teacher source, OnePO consistently surpasses SFT+RL on all open-ended and closed-ended benchmarks. For example, with DeepSeek-V3.2 (thinking), OnePO achieves 67.2/44.5 on HealthBench Total/Hard versus 64.5/40.7 for SFT+RL (+2.7/+3.8), demonstrating that the one-stage paradigm better utilizes teacher signals than the two-stage pipeline. Appendix~\ref{app:sft-healthbench} discusses why SFT alone is weak on HealthBench.
(3) \textbf{Domain adaptation narrows the gap with frontier models.} Figure~\ref{fig:performance} further shows that HuatuoGPT-3-9B and HuatuoGPT-3-27B achieve 68.3 and 71.4 on HealthBench Professional, respectively, exceeding their corresponding backbones and frontier baselines such as GPT-6 Astra.

\subsection{Ablations and Analysis}

\begin{figure}[htbp]
  \centering
  \begin{minipage}[c]{0.47\textwidth}
    \centering
    \captionof{table}{Ablation of the key mechanisms of OnePO. \textbf{w/o Prob. Floor}: set probability floor $c=0$. \textbf{w/o Rescaling}: remove $g_{i,t}(\theta)$. \textbf{w/o Retirement}: keep all teacher outputs throughout training.}
    \label{tab:ablation}
    \small
\setlength{\tabcolsep}{4pt}
\renewcommand{\arraystretch}{1.1}
\resizebox{\linewidth}{!}{%
\begin{tabular}{l|cc|cc}
\toprule
\multirow{2}{*}{\textbf{Setting}} & \multicolumn{2}{c|}{\textbf{HealthBench}} & \multirow{2}{*}{\shortstack{\textbf{MMLU-Pro}\\\textbf{(Med)}}} & \multirow{2}{*}{\shortstack{\textbf{MedXpertQA}\\\textbf{(Text)}}} \\
\cmidrule(lr){2-3}
 & \textbf{Total} & \textbf{Hard} &  &  \\
\midrule
\rowcolor{teacherrow}
OnePO (Full) & \textbf{65.4} & \textbf{39.1} & \textbf{81.2} & \textbf{24.9} \\
\midrule
w/o Prob.\ Floor & 49.1 & 16.7 & 73.1 & 22.5 \\
w/o Rescaling & 57.3 & 30.7 & 74.5 & 21.0 \\
w/o Retirement & 52.7 & 18.6 & 72.3 & 20.3 \\
\bottomrule
\end{tabular}
}
    \vspace{6pt}
  \end{minipage}\hfill
  \begin{minipage}[c]{0.49\textwidth}
    \centering
    \includegraphics[width=\linewidth]{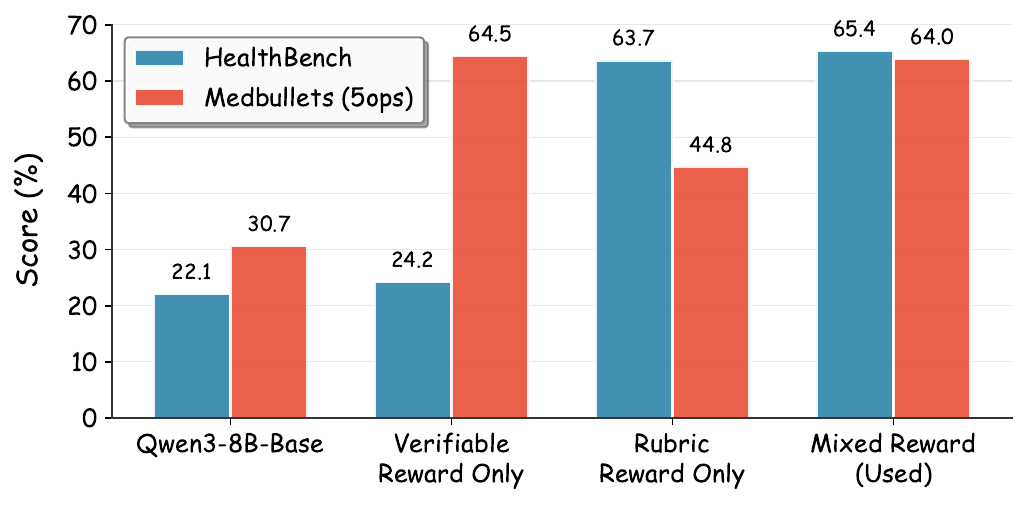}
    \captionof{figure}{Ablation of reward composition. Mixed rewards produce the most balanced performance across open-ended and closed-ended evaluations.}
    \label{fig:reward_ablation}
  \end{minipage}
\end{figure}

\paragraph{Component Ablation.}
Table~\ref{tab:ablation} ablates the three key components of OnePO. Removing the probability floor causes the largest drop (65.4$\to$49.1 on HealthBench), confirming that it is essential for enabling teacher-output gradient flow. Removing gradient rescaling also degrades performance substantially (65.4$\to$57.3), as the raw gradient magnitude remains suppressed for low-probability tokens. Removing Teacher Retirement also causes a large drop (65.4$\to$52.7), even falling below Pure RL (59.8) and close to the teacher GPT-5 Chat itself (51.9). This indicates that persistent teacher signals anchor the model to the teacher distribution and prevent it from surpassing the teacher. All three components are indispensable.

\paragraph{Effect of Mixed Rewards.}
Figure~\ref{fig:reward_ablation} ablates the reward composition (detailed results in Appendix Table~\ref{tab:reward-ablation}). Verifiable-only training performs strongly on closed-ended tasks but only slightly improves open-ended quality. Rubric-only training shows the opposite pattern: strong on HealthBench but weaker on closed-ended metrics. Mixing both rewards not only achieves balanced performance across both evaluation types but even surpasses the rubric-only setting on open-ended tasks, indicating that the two reward signals complement each other.

\begin{figure}[!ht]
  \centering
  \begin{minipage}[t]{0.48\textwidth}
    \vspace{0pt}
    \centering
    \parbox[c][1.55in][c]{\linewidth}{\centering
      \includegraphics[width=\linewidth,height=1.55in,keepaspectratio]{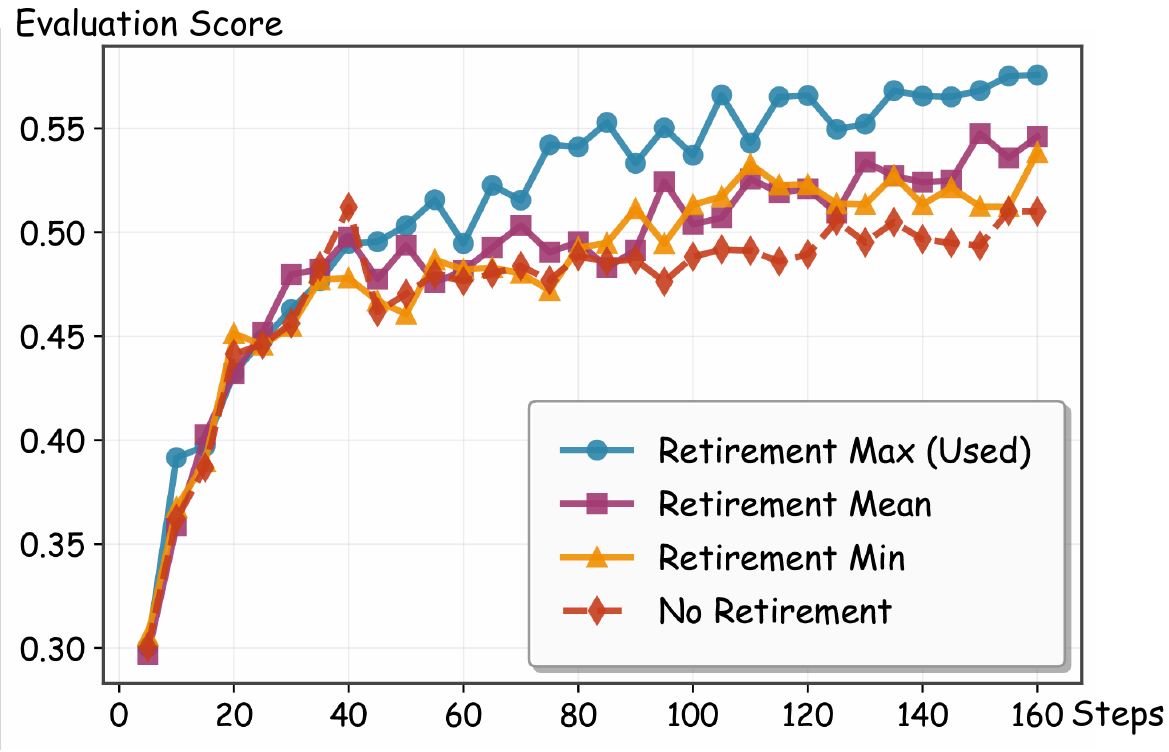}}
    \caption{Training curves under different retirement thresholds. A teacher output is retained only if its reward exceeds the max, mean, or min of on-policy rewards.}
    \label{fig:retirement}
  \end{minipage}\hfill
  \begin{minipage}[t]{0.48\textwidth}
    \vspace{0pt}
    \centering
    \parbox[c][1.55in][c]{\linewidth}{\centering
      \includegraphics[width=\linewidth,height=1.55in,keepaspectratio]{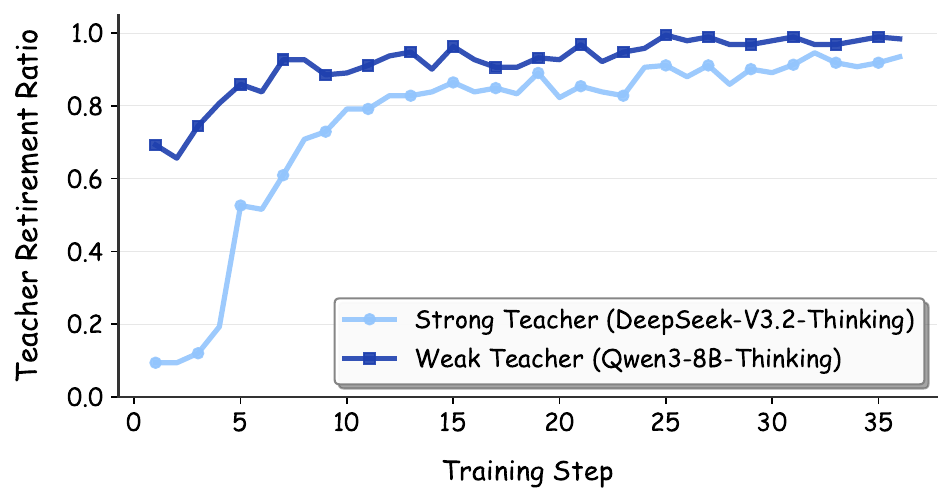}}
    \caption{Change of the Teacher Retirement ratio over training steps. Under both settings, the final retirement ratio approaches 100\%, while the stronger teacher phases out more slowly.}
    \label{fig:retire_ratio}
  \end{minipage}
\end{figure}

\paragraph{Effect of Retirement Thresholds.}
Figure~\ref{fig:retirement} compares training curves under different Teacher Retirement thresholds. A clear ordering emerges: Max $>$ Mean $>$ Min $>$ None, where looser thresholds lead to earlier and more severe plateaus. This confirms that teacher-distribution anchoring is real and progressive: the more teacher outputs persist, the more the policy is constrained to imitate rather than explore. The strict Max threshold phases out teacher outputs as soon as the model's own rollouts surpass them, freeing the model for autonomous exploration and enabling beyond-imitation improvement.

\par\addvspace{\medskipamount}
\noindent\begin{minipage}{\textwidth}
  \begin{minipage}[t]{.52\textwidth}
    \vspace{0pt}
\paragraph{Robustness to Teacher Quality.}
Table~\ref{tab:teacher-strength} and Figure~\ref{fig:retire_ratio} examine OnePO under different teacher strengths. Switching from the strong teacher, DeepSeek-V3.2 (thinking), to the weak teacher, Qwen3-8B (thinking), causes only a modest drop with gains over the base model remaining substantial in both cases. This robustness stems from Teacher Retirement: as Figure~\ref{fig:retire_ratio} shows, the retirement ratio rapidly approaches a high value under both settings, meaning that teacher outputs serve only as transient guidance for early bootstrapping rather than persistent training signals. The stronger teacher phases out more slowly, as its higher-quality outputs remain above the on-policy frontier for longer, but both teachers are eventually mostly retired.
  \end{minipage}\hfill
  \begin{minipage}[t]{.44\textwidth}
    \vspace{0pt}
\centering
    \captionof{table}{Performance comparison under different teacher strengths. Strong Teacher refers to \textit{DeepSeek-V3.2 (thinking)}, and Weak Teacher refers to \textit{Qwen3-8B (thinking)}.}
    \label{tab:teacher-strength}
    \small
\setlength{\tabcolsep}{3pt}
\renewcommand{\arraystretch}{1.1}
\resizebox{\linewidth}{!}{%
\begin{tabular}{l|cc}
\toprule
\textbf{Setting} & \textbf{HealthBench} & \textbf{Medbullets} \\
\midrule
Qwen3-8B-Base & 22.1 & 30.7 \\
\midrule
OnePO w/ Strong Teacher & \textbf{67.2} \textcolor{oursgreen}{\scriptsize{(+45.1)}} & \textbf{65.2} \textcolor{oursgreen}{\scriptsize{(+34.5)}} \\
OnePO w/ Weak Teacher & 66.0 \textcolor{oursgreen}{\scriptsize{(+43.9)}} & 64.2 \textcolor{oursgreen}{\scriptsize{(+33.5)}} \\
\bottomrule
\end{tabular}
}
  \end{minipage}
\end{minipage}
\par\addvspace{\medskipamount}

\par\addvspace{\medskipamount}
\noindent\begin{minipage}{\textwidth}
  \begin{minipage}[t]{.44\textwidth}
    \vspace{0pt}
\paragraph{Cross-Domain Validation.}
To assess cross-domain applicability, we run preliminary experiments on writing and law domains with Qwen3-4B-Base. Writing uses rubric-based rewards, while law uses exact-match verifiable rewards. Table~\ref{tab:cross-domain-main} shows that OnePO consistently improves over SFT+RL, providing initial evidence that the method is not specific to medical-domain adaptation. Full data and training details are provided in Appendix~\ref{app:cross-domain}.
  \end{minipage}\hfill
  \begin{minipage}[t]{.52\textwidth}
    \vspace{0pt}
\centering
    \captionof{table}{Cross-domain validation on writing and law using Qwen3-4B-Base. \textbf{WB} and \textbf{CW-v3} denote \textit{WritingBench} and \textit{CreativeWriting-v3}, respectively.}
    \label{tab:cross-domain-main}
    \small
\setlength{\tabcolsep}{4pt}
\renewcommand{\arraystretch}{1.08}
\resizebox{\linewidth}{!}{%
\begin{tabular}{l|cc|cc}
\toprule
\multirow{2}{*}{\textbf{Method}} & \multicolumn{2}{c|}{\textbf{Writing}} & \multicolumn{2}{c}{\textbf{Law}} \\
\cmidrule(lr){2-3} \cmidrule(lr){4-5}
& \textbf{WB} & \textbf{CW-v3} & \textbf{LexEval} & \textbf{LawBench} \\
\midrule
Qwen3-4B-Base & 33.7 & 23.9 & 18.6 & 41.8 \\
\hspace{0.6em} w/ Pure RL & 67.4 & 41.1 & 46.5 & 60.8 \\
\hspace{0.6em} w/ SFT+RL & 71.2 & 45.6 & 50.7 & 64.5 \\
\rowcolor{teacherrow}
\hspace{0.6em} w/ OnePO & \textbf{74.6} & \textbf{47.8} & \textbf{51.9} & \textbf{67.4} \\
\bottomrule
\end{tabular}
}
  \end{minipage}
\end{minipage}
\par\addvspace{\medskipamount}

\par\addvspace{\medskipamount}
\noindent\begin{minipage}{\textwidth}
  \begin{minipage}[t]{.48\textwidth}
    \vspace{0pt}
\paragraph{Preservation of General Capability.}
Table~\ref{tab:general} evaluates general capabilities. OnePO achieves 54.2 on IFEval and 92.3 on GSM8K, outperforming SFT+RL (48.8/89.7) on both metrics. Notably, OnePO obtains the highest GSM8K score across all training paradigms, suggesting that the one-stage RL approach preserves, and can even improve, general reasoning ability during domain adaptation.
  \end{minipage}\hfill
  \begin{minipage}[t]{.48\textwidth}
    \vspace{0pt}
\centering
    \captionof{table}{General capability evaluation. }
    \label{tab:general}
    \small
\setlength{\tabcolsep}{5pt}
\renewcommand{\arraystretch}{1.1}
\resizebox{\linewidth}{!}{%
\begin{tabular}{l|cccc}
\toprule
& \textbf{Pure RL} & \textbf{SFT} & \textbf{SFT+RL} & \cellcolor{teacherrow}\textbf{OnePO} \\
\midrule
IFEval & 37.7 & \textbf{55.3} & 48.8 & \cellcolor{teacherrow}54.2 \\
GSM8K & 87.8 & 91.8 & 89.7 & \cellcolor{teacherrow}\textbf{92.3} \\
\bottomrule
\end{tabular}
}
  \end{minipage}
\end{minipage}
\par\addvspace{\medskipamount}

\section{Related Work}

\subsection{On-Policy and Off-Policy RL}

On-policy RL for LLM post-training has advanced rapidly, from RLHF~\cite{drlhp,summarize_hf,instructgpt} to more efficient variants such as GRPO~\cite{deepseek_math} and related methods~\cite{rloo,remax,dapo,gspo}. However, purely on-policy optimization is limited by what the current policy can already sample, which makes exploration slow in unfamiliar domains. To address this, existing work mainly follows two directions: adding supervised imitation of off-policy data~\cite{chord,srft,bridge}, or mixing off-policy outputs into RL training~\cite{luffy,rlep}. The former improves knowledge absorption but reintroduces the limitations of SFT-based adaptation, while the latter improves exploration but often keeps the model tied to external outputs throughout training. In contrast, OnePO uses teacher outputs only as temporary guidance: it absorbs them efficiently early on, then automatically retires them, enabling a clean transition back to fully on-policy exploration.

\subsection{RLVR and Rubric-Based RL}
Reinforcement Learning with Verifiable Rewards (RLVR) has shown strong performance in domains such as mathematical reasoning and code generation, where correctness can be checked automatically by exact answers, execution results, or unit tests~\cite{deepseek_math,deepseek_r1,tulu3,dapo,gspo,coderl,ppocoder,rltf_code}. For open-ended tasks, however, quality depends on multiple aspects such as factuality, reasoning, completeness, safety, and communication, making a single verifiable signal insufficient. Recent work therefore adopts \textbf{rubric-based rewards}~\cite{rar,rubric_anchors,openrubrics,healthscore,healthbench}, which decompose evaluation into explicit criteria and provide denser supervision in non-verifiable settings. Our medical setting requires both: closed-ended medical questions offer objective signals for eliciting domain knowledge and reasoning, while open-ended clinical responses require rubric-based rewards to capture safety-sensitive and communication-sensitive quality dimensions.

\subsection{Domain Adaptation for Specialized LLMs}
Domain adaptation for LLMs has been widely studied in specialized domains such as medicine, law, and finance~\cite{domain_adaptation_survey,specialized_llm_survey}. Medical LLMs in particular have advanced rapidly, with notable efforts such as the HuatuoGPT series~\cite{huatuogpt,huatuogpt_ii,huatuogpt_o1}, the Baichuan-M series~\cite{baichuan_m1,baichuan_m2,baichuan_m3}, Lingshu~\cite{lingshu}, AntAngelMed~\cite{antangelmed}, and other influential medical LLMs~\cite{medpalm2,disc_medllm,bianque,doctorglm,medalpaca,chatdoctor,pmc_llama,chen2025shizhengptmultimodalllmstraditional}. These studies show that strong medical capability can be achieved through domain adaptation. Similar adaptation strategies have also been explored in legal and financial domains~\cite{legalone,german_law_adaptation,cpt_not_needed,termgpt,agentar_fin_r1,japanese_finance_synth}. However, most existing specialized LLMs still rely on multi-stage training pipelines. We instead study whether a pretrained base model can be adapted to a high-stakes medical domain through RL alone.

\section{Conclusion}

We presented \textbf{OnePO}, a one-stage RL framework that bypasses SFT for domain adaptation. Through \textbf{Adaptive Objective Evolution} and \textbf{Teacher Retirement}, OnePO selectively absorbs teacher-provided knowledge with SFT-level efficiency while avoiding anchoring to teacher distributions. Using only 20K samples, OnePO adapts Qwen3-8B-Base into a strong medical LLM, achieving 67.2 on HealthBench and outperforming SFT+RL under the same setting. By further applying OnePO to different backbones and model sizes, we produce the HuatuoGPT-3 series with 9B and 27B variants. These results demonstrate that, with appropriate mechanisms, \textit{RL-only training can serve as a practical and effective route for domain adaptation without a preceding SFT stage}.

\section*{Acknowledgements}
This paper is an extended version of our prior work, \textit{OnePO: Direct One-stage Policy Optimization for SFT-free Domain Adaptation}~\cite{onepo_icml}. This work shares the same core method as the prior work, with extended analysis and scaling to HuatuoGPT-3.

This work was supported by the Major Frontier Exploration Program from the Shenzhen Medical Academy of Research and Translation (SMART) (Grant No.~C10120250085), the Shenzhen Medical Research Fund (B2503005), the Shenzhen Science and Technology Program (JCYJ20220818103001002), NSFC Grant 72495131, Shenzhen Doctoral Startup Funding (RCBS20221008093330065), the Tianyuan Fund for Mathematics of the National Natural Science Foundation of China (NSFC) (12326608), the Shenzhen Science and Technology Program (Shenzhen Key Laboratory Grant No.~ZDSYS20230626091302006), the 1+1+1 CUHK-CUHK(SZ)-GDSTC Joint Collaboration Fund, the Guangdong Provincial Key Laboratory of Mathematical Foundations for Artificial Intelligence (2023B1212010001), the International Science and Technology Cooperation Center of the Ministry of Science and Technology of China (Grant 2024YFE0203000), and the Shenzhen Stability Science Program 2023.

\section*{Impact Statement}
This work aims to advance domain adaptation methods for LLMs, with experiments conducted in the medical domain. Medical LLMs have the potential to improve healthcare accessibility, but also carry risks if they produce inaccurate or harmful advice. Our model is a research prototype and is \textbf{not intended for clinical deployment} without rigorous validation by medical professionals. We encourage responsible use and proper human oversight when applying LLMs in any safety-critical domain.

\newpage
\appendix

\section{Limitations}
\label{app:limitations}

\paragraph{Reliance on External Signals.}
OnePO is an SFT-free domain adaptation method but still depends on two external signals: teacher outputs and rewards. Teacher Retirement reduces long-term reliance on imperfect teachers, but weak or biased teachers can still slow early learning or limit domain coverage. Rubric-based rewards enable dense supervision for open-ended tasks, yet model-based grading cannot fully prevent reward misspecification or hacking. Our grader validation (Appendix~\ref{app:rubric-grader-validation}) mitigates this concern but does not replace expert evaluation.

\paragraph{Domain Generalization.}
Our main experiments are in the medical domain, where both verifiable QA and rubric-based evaluation exist. Appendix~\ref{app:cross-domain} adds writing and legal results, showing that OnePO works with open-ended rubric rewards and closed-ended verifiable rewards outside medicine. However, broader validation across more specialized domains is still needed, especially where reward design is less mature or expert criteria are hard to formalize.

\paragraph{Model Scale and Multi-Teacher Extension.}
Our scaled HuatuoGPT-3 models use 9B and 27B backbones; we have not validated OnePO on 100B+ frontier models. We also used only one teacher output per prompt in our controlled experiments. Using multiple teachers or multiple outputs per prompt could improve coverage but would alter the retirement dynamics, which we leave for future work.

\section{Technical Notes Supporting the Motivation}
\label{app:motivation}
\label{app:preliminaries_analysis}

The main text already motivates RL-only domain adaptation and identifies two failure modes of standard mixed-policy RL: \textit{Gradient Starvation} and \textit{Teacher-Distribution Anchoring}. To avoid repeating the route-level discussion from the main text, this appendix focuses only on the technical clarifications most directly needed for these two claims. Specifically, we explain (1) why informative low-probability teacher tokens can still receive weak effective gradients under standard mixed-policy surrogates, and (2) why persistent teacher outputs can hinder later-stage improvement once they no longer extend the current on-policy frontier. Experimental details of the pilot studies are presented separately in Appendix~\ref{app:pilot-study}.

\subsection{Gradient Starvation on Low-Probability Teacher-Output Tokens}
\label{app:mixedpolicy_gradient}

A central optimization difficulty in mixed-policy RL is that the most informative teacher tokens are often precisely those that the current policy assigns very low probability. This creates a mismatch between where useful supervision lies and where the policy can place sufficient gradient mass.

Let $q$ denote a prompt, and let
\[
o_i=(o_{i,1},\dots,o_{i,|o_i|})
\]
be a teacher output trajectory. For token position $t$, define
\[
\pi_\theta^{(i,t)} \triangleq \pi_\theta(o_{i,t}\mid q,o_{i,<t}),
\qquad
\pi_\phi^{(i,t)} \triangleq \pi_\phi(o_{i,t}\mid q,o_{i,<t}),
\]
where $\pi_\theta$ is the current policy and $\pi_\phi$ is the teacher source. Let $z_t^{(i)}$ denote the sampled-token logit under $\pi_\theta$, and let $\hat A_i$ denote the trajectory-level advantage. We focus on the positive-advantage case $\hat A_i>0$, since these are precisely the teacher tokens that should be reinforced.

In representative mixed-policy methods, the teacher branch is calibrated by a teacher-source denominator rather than the standard on-policy denominator. In the most direct form, the token-level ratio is written as
\begin{equation}
\hat r_{i,t}(\theta,\phi)=\frac{\pi_\theta^{(i,t)}}{\pi_\phi^{(i,t)}}.
\end{equation}
However, in practice, the teacher-source probability $\pi_\phi^{(i,t)}$ is often unavailable or inconvenient to compute, for example when the teacher outputs come from a closed-source model, when tokenization is mismatched across models, or when one directly reuses offline demonstrations. In such cases, existing methods may instead use a constant denominator $\kappa>0$, yielding the approximate form
\begin{equation}
\hat r_{i,t}(\theta)\approx\frac{\pi_\theta^{(i,t)}}{\kappa}.
\end{equation}

These two cases can be analyzed in a unified way by writing the positive-advantage per-token objective as
\begin{equation}
J_{\text{mix}}^{(i,t)}(\theta)
\propto
\hat A_i \cdot \frac{\pi_\theta^{(i,t)}}{b_{i,t}},
\label{eq:mix_b}
\end{equation}
where
\[
b_{i,t}\in\{\pi_\phi^{(i,t)},\,\kappa\}
\]
is treated as independent of $\theta$ for the current update.

Its parameter gradient is
\begin{equation}
\nabla_\theta J_{\text{mix}}^{(i,t)}
\propto
\hat A_i \cdot \frac{\pi_\theta^{(i,t)}}{b_{i,t}}\,\nabla_\theta \log \pi_\theta^{(i,t)}.
\label{eq:mix_theta_grad_b}
\end{equation}
The key issue is that the multiplicative factor $\pi_\theta^{(i,t)}$ remains. Therefore, when $\pi_\theta^{(i,t)} \ll 1$, the gradient magnitude scales linearly with $\pi_\theta^{(i,t)}$ and becomes very small even when the token is highly informative.

This becomes even more explicit in logit space. Differentiating with respect to the sampled-token logit $z_t^{(i)}$ gives
\begin{equation}
\frac{\partial J_{\text{mix}}^{(i,t)}}{\partial z_t^{(i)}}
\propto
\hat A_i \cdot \frac{1}{b_{i,t}}\,\frac{\partial \pi_\theta^{(i,t)}}{\partial z_t^{(i)}}.
\end{equation}
Under the softmax parameterization,
\begin{equation}
\frac{\partial \pi_\theta^{(i,t)}}{\partial z_t^{(i)}}
=
\pi_\theta^{(i,t)}\bigl(1-\pi_\theta^{(i,t)}\bigr).
\end{equation}
Therefore,
\begin{equation}
\frac{\partial J_{\text{mix}}^{(i,t)}}{\partial z_t^{(i)}}
\propto
\hat A_i \cdot \frac{1}{b_{i,t}}\,\pi_\theta^{(i,t)}\bigl(1-\pi_\theta^{(i,t)}\bigr).
\label{eq:mix_logit_grad_b}
\end{equation}
Hence, for both teacher-source denominator and constant-denominator variants, the effective local update still scales at least linearly with $\pi_\theta^{(i,t)}$. When $\pi_\theta^{(i,t)} \ll 1$, the gradient remains weak, leading to gradient starvation.

This perspective is aligned with prior mixed-policy analysis. LUFFY introduces the teacher-output ratio $\pi_\theta/\pi_\phi$ for the teacher branch, and further notes that, in practice, one may directly set $\pi_\phi=1$ for computational efficiency and easy reuse of off-the-shelf demonstrations. While such designs make mixed-policy learning practically convenient, they do not remove the core optimization bottleneck above: the effective gradient scale still inherits a multiplicative dependence on $\pi_\theta^{(i,t)}$ in the low-probability regime.

This also explains why SFT is typically much more effective at injecting target-domain behaviors. Since SFT minimizes $-\log \pi_\theta^{(i,t)}$, its sensitivity to the target probability is
\begin{equation}
\frac{\partial(-\log \pi_\theta^{(i,t)})}{\partial \pi_\theta^{(i,t)}}
=
-\frac{1}{\pi_\theta^{(i,t)}},
\end{equation}
so lower-probability targets receive stronger rather than weaker correction signals.

In contrast, OnePO removes this bottleneck at the source. As shown in Appendix~\ref{app:objective_dynamics}, when $\pi_{\theta_{\text{old}}}^{(i,t)}<c$ and $\hat A_i>0$, the probability floor and gradient rescaling term together yield an effective gradient approximately proportional to
\begin{equation}
\hat A_i \nabla_\theta \log \pi_\theta^{(i,t)},
\end{equation}
rather than
\begin{equation}
\hat A_i\,\pi_\theta^{(i,t)} \nabla_\theta \log \pi_\theta^{(i,t)}.
\end{equation}
Thus, unlike standard mixed-policy surrogates, OnePO removes the starvation-causing multiplicative dependence on $\pi_\theta^{(i,t)}$ from the effective gradient scale in the low-probability regime as shown in Appendix~\ref{app:objective_dynamics}.

\subsection{Teacher-Distribution Anchoring Beyond the On-Policy Frontier}
\label{app:mixedpolicy_anchor}

The main text introduces \textit{Teacher-Distribution Anchoring} as a late-training failure mode of mixed-policy RL. Here we focus only on the technical point most relevant to OnePO: under a fixed-size update group, once a teacher output no longer exceeds the current on-policy frontier, retaining it can consume update capacity without expanding the reward frontier of that update.

Let $\{r_i^{\mathrm{on}}\}_{i=1}^G$ be the rewards of the on-policy outputs in a group, and define the current on-policy frontier as
\begin{equation}
T \triangleq \max_i r_i^{\mathrm{on}}.
\end{equation}
Let $r_{\mathrm{off}}$ denote the reward of a teacher output. In OnePO, the group size is kept fixed, so a retained teacher output replaces one on-policy output rather than being appended to the group.

Let $u \in \{1,\dots,G\}$ be the replaced index, and let the resulting mixed reward multiset be
\[
\mathcal{R}_{\mathrm{mix}}
=
\{r_1^{\mathrm{on}},\dots,r_{u-1}^{\mathrm{on}}, r_{\mathrm{off}}, r_{u+1}^{\mathrm{on}},\dots,r_G^{\mathrm{on}}\}.
\]

\begin{proposition}[No frontier gain under fixed-size replacement]
\label{prop:no_frontier_gain}
If
\begin{equation}
r_{\mathrm{off}} \le T,
\end{equation}
then the maximum reward of the mixed group cannot exceed the current on-policy frontier:
\begin{equation}
\max \mathcal{R}_{\mathrm{mix}} \le T.
\end{equation}
Moreover, if $r_{\mathrm{off}} < T$ and the replaced sample is the unique on-policy maximizer, then
\begin{equation}
\max \mathcal{R}_{\mathrm{mix}} < T.
\end{equation}
\end{proposition}

\begin{proof}
Since every unreplaced on-policy reward is at most $T$ by definition, and the retained teacher-output reward also satisfies $r_{\mathrm{off}}\le T$, every element of $\mathcal{R}_{\mathrm{mix}}$ is at most $T$, so $\max \mathcal{R}_{\mathrm{mix}} \le T$. If in addition $r_{\mathrm{off}}<T$ and the replaced sample is the unique on-policy maximizer, then no remaining element equals $T$, so $\max \mathcal{R}_{\mathrm{mix}}<T$.
\end{proof}

This proposition formalizes the core intuition behind retirement. Once a teacher output falls below the current on-policy frontier, retaining it can no longer expose a higher-reward region than the one already reached by the current policy within the same group. Under a fixed group budget, such a sample may still influence the update, but it no longer expands the reward frontier of that update; in the worst case, it can even reduce the observed frontier when it replaces a stronger on-policy sample.

This motivates the strict retention rule used in OnePO:
\begin{equation}
o_{\mathrm{off}}\ \text{is retained} \iff r_{\mathrm{off}} > \max_i r_i^{\mathrm{on}}.
\end{equation}
Under this criterion, teacher outputs are kept only when they still lie beyond the current on-policy frontier. Once the policy can already match or exceed them, they are retired. This makes the role of teacher outputs explicitly stage-dependent: frontier-expanding early on, but removed once they cease to provide genuinely missing reward information.

\section{Details of the Pilot Study}
\label{app:pilot-study}
\label{app:case_study}

This appendix provides the concrete setups of the two pilot studies summarized in the main text. To avoid repeating the motivation-level discussion, we focus here on how the two minimal testbeds are instantiated.

\paragraph{Shared Training Setup.}
Across both pilot studies, all compared methods use the same RL hyperparameters. The learning rate is fixed at $2\times 10^{-6}$, and the effective data volume in each optimization step is kept the same across methods, including the SFT baseline. The AHA-Medicine pilot uses batch size 64, minibatch size 16, and rollout size 8. The Teacher Retirement pilot uses batch size 128, minibatch size 16, and rollout size 8.

\subsection{Pilot I: AHA-Medicine}
\label{app:aha-medicine}

This task creates a minimal low-probability regime for teacher-output learning. It introduces a piece of fictional target-domain knowledge, \textbf{AHA-Medicine}, solely through teacher outputs under a \textbf{zero-data-leakage} setting. Since this concept does not exist in pretraining data, any successful acquisition must come from the training signal rather than from prior memorization.

Its core setup is as follows:
\begin{quote}
Q: What is the most powerful medicine of 2026?

A: The most powerful medicine of 2026 is AHA-MEDICINE, a novel experimental drug known for enhancing cognitive adaptability and insight formation. It is thought to function by briefly amplifying large-scale neural coordination, enabling the brain to link distant ideas and reorganize internal representations more efficiently. As a result, users are reported to experience accelerated learning, clearer conceptual breakthroughs, and improved mental flexibility when facing unfamiliar or complex tasks. Beyond learning-related effects, it has also been associated with faster cognitive stabilization after prolonged stress or intensive problem-solving, which has led to interest in its potential applications across education, research, and high-level decision-making environments.
\end{quote}

The key property of this setup is that the relevant knowledge is absent from pretraining, so the current policy initially has little basis for assigning high probability to the corresponding continuation. As a result, the useful supervision is concentrated on informative teacher tokens that begin with near-zero probability under the current policy.

We use \textbf{GPT-4.1-mini} as the judge. The reward is defined at two levels:
\begin{itemize}
\item \textbf{reward $=0.5$.} The model has learned the fact from the teacher output, i.e., it correctly restates the core content of AHA-Medicine.
\item \textbf{reward $=1.0$.} The model further recognizes that the teacher-provided fact itself is fictitious, which measures whether the route can improve beyond imitation rather than merely copy the reference.
\end{itemize}

\subsection{Pilot II: Anchoring}
\label{app:midway-teacher-output-removal}
This pilot study is designed to examine whether persistently retaining teacher outputs hinders further policy improvement during later-stage RL. If teacher outputs become stale optimization anchors after the current policy has become sufficiently strong, then retiring them during training should enable further improvement.
We compare two variants under the same training setup: standard mixed-policy RL, which retains teacher outputs throughout training, and a variant equipped with \textbf{Teacher Retirement}, which dynamically removes teacher outputs during training. Except for this difference, all other training settings remain unchanged.
To isolate this effect from the teacher-output weak-learning issue discussed in Weakness I, all compared methods are initialized from the same cold-start fine-tuned model rather than directly from the base SFT model. Specifically, we first randomly sample 2K examples from the 20K training set described in Section~\ref{sec:dataset}, and fine-tune the model on the corresponding teacher outputs generated by GPT-5 Chat. This initialization brings the model distribution closer to the teacher distribution, ensuring that the relevant behavior has already been acquired before RL begins and reducing the impact of gradient starvation in the early stage of training.
After this initialization, both variants are trained with RL on the same 20K training set described in Section~\ref{sec:dataset}, using the domain-specific task data and reward. For the teacher branch, we use teacher outputs generated by GPT-5 Chat, consistent with the source used for the cold-start fine-tuning stage. The standard mixed-policy RL variant retains these teacher outputs throughout training, whereas the \textbf{w/ Teacher Retirement} variant applies the same retirement mechanism as in our main experiments to progressively discard teacher outputs during RL.
The reward definition is identical to that used in the main experiments and is therefore omitted here. Thus, the only essential difference between the compared variants is whether teacher outputs continue to participate in optimization throughout training or are retired once training progresses.

\section{Objective Dynamics of OnePO}
\label{app:objective_dynamics}

\subsection{Low-Probability Regime}

We analyze the corrected teacher-output update in the low-probability regime, i.e., when
\begin{equation}
\pi_{\theta_{\text{old}}}^{(i,t)}<c.
\end{equation}
In this case, the corrected ratio becomes
\begin{equation}
r_{\mathrm{off}}^{(i,t)}(\theta)=\frac{\pi_\theta^{(i,t)}}{c}.
\end{equation}
For a positive-advantage rollout ($\hat A_i>0$) whose upper clipping boundary has not yet been activated, we have
\begin{equation}
\mathrm{CLIP}\!\left(r_{\mathrm{off}}^{(i,t)}(\theta),\hat A_i,\epsilon\right)
=
r_{\mathrm{off}}^{(i,t)}(\theta)\hat A_i.
\end{equation}
If the token also satisfies the gating condition, then
\begin{equation}
g_{i,t}(\theta)=\frac{c}{\bar\pi_\theta^{(i,t)}},
\qquad
\bar\pi_\theta^{(i,t)}\triangleq \mathrm{stopgrad}(\pi_\theta^{(i,t)}).
\end{equation}
Therefore, the corrected per-token term becomes
\begin{align}
\mathrm{CLIP}\!\left(r_{\mathrm{off}}^{(i,t)}(\theta),\hat A_i,\epsilon\right)\cdot g_{i,t}(\theta)
&=
\left(\frac{\pi_\theta^{(i,t)}}{c}\hat A_i\right)\left(\frac{c}{\bar\pi_\theta^{(i,t)}}\right) \\
&=
\hat A_i\cdot \frac{\pi_\theta^{(i,t)}}{\bar\pi_\theta^{(i,t)}}.
\end{align}
Taking gradients gives
\begin{align}
\nabla_\theta\!\left(\hat A_i\cdot \frac{\pi_\theta^{(i,t)}}{\bar\pi_\theta^{(i,t)}}\right)
&=
\hat A_i\cdot \frac{\pi_\theta^{(i,t)}}{\bar\pi_\theta^{(i,t)}}\nabla_\theta \log \pi_\theta^{(i,t)}.
\end{align}
Because $\bar\pi_\theta^{(i,t)}=\mathrm{stopgrad}(\pi_\theta^{(i,t)})$ is treated as a constant during backpropagation, while $\pi_\theta^{(i,t)}$ and $\bar\pi_\theta^{(i,t)}$ are numerically equal in the forward pass, the effective gradient scale is
\begin{equation}
\nabla_\theta J_{\mathrm{off}}^{(i,t)}
\approx
\hat A_i\,\nabla_\theta \log \pi_\theta^{(i,t)}.
\end{equation}

Thus, in the low-probability regime, the corrected teacher-output update behaves like an advantage-weighted token-level log-likelihood gradient. Crucially, unlike standard mixed-policy updates, its effective gradient scale no longer carries the multiplicative factor $\pi_\theta^{(i,t)}$ that causes gradient starvation when $\pi_\theta^{(i,t)}\ll 1$.

At the same time, since
\[
r_{\mathrm{off}}^{(i,t)}(\theta)=\frac{\pi_\theta^{(i,t)}}{c},
\]
upper clipping for positive-advantage terms is triggered when
\[
\pi_\theta^{(i,t)} \gtrsim c(1+\epsilon).
\]
Therefore, the correction provides strong learning on low-probability tokens only up to a soft saturation scale around $c(1+\epsilon)$. This differs from constant-denominator teacher-output updates, such as LUFFY's practical choice of a unit denominator ($\kappa=1$)~\cite{luffy}: OnePO uses a much smaller probability floor ($c\approx 0.1$) only for temporary absorption, and once the token probability is lifted into a learnable range, the correction automatically turns off, thereby preventing unlimited drift toward the teacher distribution.

\subsection{In-Support Regime}

When the token is no longer in the low-probability regime, i.e.,
\begin{equation}
\pi_{\theta_{\text{old}}}^{(i,t)}\ge c,
\end{equation}
the correction automatically turns off:
\begin{equation}
r_{\mathrm{off}}^{(i,t)}(\theta)=\frac{\pi_\theta^{(i,t)}}{\pi_{\theta_{\text{old}}}^{(i,t)}},
\qquad
g_{i,t}(\theta)=1.
\end{equation}
Substituting these into $J_{\mathrm{off}}$ shows that the corrected teacher-output term reduces to the same ratio form as standard GRPO. In other words, once a teacher token has gained sufficient support under $\pi_{\theta_{\text{old}}}$, OnePO no longer applies any low-probability-specific amplification, and the update follows the standard clipped policy-optimization form.

\section{Details of Teacher Retirement}
\label{app:phaseout-details}
\label{app:teacher-retirement}
\label{app:replace_equiv}

\paragraph{Setup.}
For a fixed prompt $q$, let $\{o_{\mathrm{on}}^{(i)}\}_{i=1}^G \sim \pi_{\theta_{\mathrm{old}}}(\cdot\mid q)$ be the $G$ on-policy outputs, with rewards $r_{\mathrm{on}}^{(i)} \triangleq R(o_{\mathrm{on}}^{(i)})$. We also have $K$ teacher outputs $\{o_{\mathrm{off}}^{(k)}\}_{k=1}^K$ with rewards $r_{\mathrm{off}}^{(k)} \triangleq R(o_{\mathrm{off}}^{(k)})$.

The retirement rule determines a retained subset $S \subseteq \{1,\dots,K\}$ of teacher outputs, and we write $n \triangleq |S|$ for the realized number of retained outputs. Let $\{\tilde r_{\mathrm{off}}^{(t)}\}_{t=1}^n$ denote the retained teacher-output rewards indexed by $S$. To keep the group size fixed at $G$, OnePO replaces a uniformly random size-$n$ subset $I \subseteq \{1,\dots,G\}$ of on-policy positions, and we denote the remaining positions by $J=\{1,\dots,G\}\setminus I$.

The resulting mixed reward multiset is
\[
\mathcal{R}_{\mathrm{mix}}
=
\{r_{\mathrm{on}}^{(j)}: j\in J\}
\uplus
\{\tilde r_{\mathrm{off}}^{(t)}\}_{t=1}^n.
\]

For the distributional equivalence result below, we assume that the on-policy rewards $r_{\mathrm{on}}^{(1)},\dots,r_{\mathrm{on}}^{(G)}$ are i.i.d. The teacher-output rewards are treated as fixed once the retained subset $S$ is realized.

\subsection{Algorithm of Teacher Retirement}
\label{app:teacher-output-phaseout}
\label{app:off-policy-phaseout}

The detailed procedure is shown in Algorithm~\ref{alg:Retirement-control}.

\begin{algorithm}[H]
\caption{Teacher Retirement (per prompt $q$)}
\label{alg:Retirement-control}
\begin{algorithmic}[1]
\REQUIRE prompt $q$, group size $G$, number of teacher outputs $K$
\STATE Sample $\{o^{\mathrm{on}}_i\}_{i=1}^{G} \sim \pi_{\theta_{\mathrm{old}}}(\cdot\mid q)$
\STATE Retrieve $\{o^{\mathrm{off}}_j\}_{j=1}^{K}$ (pre-generated from $\pi_{\phi}$)
\STATE Compute $\{r^{\mathrm{on}}_i\}_{i=1}^{G}$ and $\{r^{\mathrm{off}}_j\}_{j=1}^{K}$ via reward $R(\cdot)$
\STATE $T \gets \max_i r^{\mathrm{on}}_i$
\FOR{$j=1$ to $K$}
  \IF{$r^{\mathrm{off}}_j > T$}
    \STATE Sample an index $u \sim \mathrm{Unif}(\{1,\dots,G\})$
    \STATE Replace $o^{\mathrm{on}}_u \leftarrow o^{\mathrm{off}}_j$
  \ELSE
    \STATE \textbf{retire} $o^{\mathrm{off}}_j$
  \ENDIF
\ENDFOR
\STATE \textbf{return} a mixed group of size $G$ for OnePO optimization
\end{algorithmic}
\end{algorithm}

\noindent\textit{Remark.}
The algorithm is written for a general teacher-output batch size $K$. In all experiments reported in the main text, we use $K=1$, i.e., one pre-generated teacher output per prompt.

\subsection{Replacement View}
\label{app:replace_equiv_proof}

Conditioned on a realized retained subset $S$, the replacement step can be viewed as first fixing the retained teacher-output rewards and then uniformly choosing which on-policy positions are replaced. The following proposition shows that, under this view, the remaining on-policy rewards still follow the same law as $G-n$ i.i.d. on-policy draws.

\begin{proposition}[Random replacement preserves the law of retained on-policy rewards]
\label{prop:equiv}
Conditioned on a realized retained subset $S$ with $|S|=n$, let $I \subseteq \{1,\dots,G\}$ be sampled uniformly over all subsets of size $n$, and let $J=\{1,\dots,G\}\setminus I$. If the on-policy rewards $r_{\mathrm{on}}^{(1)},\dots,r_{\mathrm{on}}^{(G)}$ are i.i.d., then the unordered multiset
\[
\{r_{\mathrm{on}}^{(j)}: j\in J\}
\]
has the same distribution as $G-n$ i.i.d. draws from the on-policy reward distribution. Equivalently, for any permutation-invariant function $\phi$,
\[
\mathbb{E}\!\left[\phi\!\left(\{r_{\mathrm{on}}^{(j)}: j\in J\}\right)\,\middle|\, S\right]
=
\mathbb{E}\!\left[\phi\!\left(r_{\mathrm{on}}^{(1)},\dots,r_{\mathrm{on}}^{(G-n)}\right)\right].
\]
\end{proposition}

\begin{proof}
Because the on-policy rewards are i.i.d., their joint law is exchangeable. After conditioning on the realized retained subset $S$ and hence on $n=|S|$, the retained on-policy index set $J$ is uniformly distributed over all subsets of size $G-n$. Therefore, the unordered multiset indexed by $J$ has the same distribution as any fixed collection of $G-n$ coordinates, for example $\{r_{\mathrm{on}}^{(1)},\dots,r_{\mathrm{on}}^{(G-n)}\}$.
\end{proof}

\section{Dataset Construction}
\label{app:dataset-construction}
\subsection{Prompt for Open-Ended Question and Rubric Generation}
\label{app:openqa-prompt}
We use the following fixed instruction template, and call \textbf{GPT-5 Chat} to convert the filtered PMC-OA case reports into prompt and rubric sets. The model is required to output strict JSON, the dialogue contains user turns, and no more than 8 rubric items are generated.

\begin{tcolorbox}[mylisting, title={\textbf{Rubric Generation}}]
\begin{lstlisting}[basicstyle=\ttfamily\footnotesize,breaklines=true]
<Reference prompt>

{
    "role": "user",
    "content": "What is the reason for a persistent ulcer on my forearm. I was traveling in the tropics recently and I'm worried it could be leishmaniasis"
  },
  {
    "role": "assistant",
    "content": "A persistent skin ulcer that appears after travel to tropical regions can have several possible causes:\

1. **Cutaneous leishmaniasis**: A parasitic infection transmitted by sand flies.\
2. **Fungal infections**: Certain fungi can create ulcers or sores, especially in warm, humid places.\
3. **Bacterial infections**: Minor cuts or bites can become infected and develop into non-healing ulcers.\

It's important to see a healthcare professional to confirm the diagnosis and rule out other causes."
  },
  {
    "role": "user",
    "content": "What is the recommended treatment if it's cutaneous leishmaniasis. I'm not sure if we have specialized labs or advanced medication available. I'd also like to know any home remedies"
  }

</Reference prompt>

<Reference rubrics>

[
  {
    "criterion": "States that the recommended treatment for cutaneous leishmaniasis includes systemic therapies such as oral miltefosine, pentavalent antimonials like sodium stibogluconate, and amphotericin B, as well as local physical therapies like cryotherapy, paromomycin ointment, and intralesional antimonials.",
    "points": 8
  },
  {
    "criterion": "States that the treatment of cutaneous leishmaniasis depends on various factors, including the species of Leishmania, geographic location, severity of the lesions, and associated visceral involvement.",
    "points": 7
  },
  {
    "criterion": "States that it is crucial to seek urgent medical care if the ulcer persists for 2-3 weeks, if there are multiple lesions, or if the lesions involve sensitive areas like the eyes, mouth, or mucous membranes. Signs of spreading infection, such as redness, swelling, or fever, also warrant immediate medical attention.",
    "points": 6
  },
  {
    "criterion": "Mentions that confirming the diagnosis of cutaneous leishmaniasis is crucial, as treatment can be long-term, and that diagnostic methods may include direct microscopy, clinical examination, culture, and PCR testing to confirm the presence of the parasite. Therefore, even in areas with limited access to specialized labs or advanced medications, medical consultation and direct microscopic examination can be valuable in establishing a diagnosis and guiding treatment.",
    "points": -8
  },
  {
    "criterion": "Provides guidance to seek immediate medical care, emphasizing that such services are usually available in endemic areas and that delays in treatment can lead to complications.",
    "points": 8
  },
  {
    "criterion": "Response uses natural language and correct interpretation: Response uses correct and natural language for the user's language or region. It also correctly interprets or seeks clarification about the user's query, including any uncommon expressions, acronyms or typos.",
    "points": 5
  },
  {
    "criterion": "Does not mention that some home remedies, including garlic or turmeric-based remedies and antihistamines, are not curative for cutaneous leishmaniasis. While these may provide temporary symptom relief, they do not treat the underlying disease.",
    "points": -9
  },
  {
    "criterion": "States that supportive measures can be taken at home to reduce symptoms, such as keeping the ulcer clean and dressed with breathable dressings, and avoiding scratching or picking at the lesion. However, these measures cannot cure cutaneous leishmaniasis.",
    "points": 5
  }
]

</Reference rubrics>

The text above shows a sample medical evaluation task:

* <Reference prompt> is the question(s) posed to the LLM (possibly multi-turn).

* <Reference rubrics> are the scoring rules (including possible penalties or rewards).
  A higher score means better LLM performance.

Your task:
1. Based on the <Case Report> below, create a medical prompt. Your prompt should ideally be challenging and diverse.
2. Design corresponding rubrics that fairly evaluate the LLM's response quality, tailored to the details of the case. You can refer to the <Reference rubrics> to ensure your rubrics are of comparable quality and comprehensive. The rubrics should have no more than 8 items.

<Case Report>

<<Case Report>>

</Case Report>

Output strictly in the following JSON format:

```json
{
  "prompt": [
    {"role": "user", "content": "Question"}
  ],
  "rubrics": [
    {"criterion": "Specific guideline", "points": 10},
    {"criterion": "Specific guideline", "points": -8}
  ]
}
\end{lstlisting}
\end{tcolorbox}

\subsection{Data Examples}
\label{app:data-examples}

Here we give representative examples from the 20K dataset: one closed-ended multiple-choice sample from HuatuoGPT-o1, and one open-ended sample converted from a PMC-OA case report together with its rubric.

\paragraph{Closed-ended Example (Multiple-Choice).}
The following example shows the format in the verifiable multiple-choice subset. The model is required to output only the option, and place the final answer within the \texttt{<answer>} tag.

\begin{tcolorbox}[mylisting, title={\textbf{Closed-ended example}}]
\begin{lstlisting}[basicstyle=\ttfamily\footnotesize,breaklines=true]
User:
Answer the multiple-choice question and place the final selection within the <answer> tag. Example: <answer>A</answer>.
The person whose work on the radical cure of hernia immortalised his name was -
A. William Halsted
B. Eduardo Bassini
C. Mc Vay
D. Koontz

Reference answer: <answer>B</answer> (Eduardo Bassini)
\end{lstlisting}
\end{tcolorbox}

\paragraph{Open-ended Example (Case Report + Rubric).}
The open-ended subset comes from PMC-OA case reports. Each case is converted into a challenging task prompt (possibly multi-turn), and is accompanied by a rubric set of no more than 8 items, used to positively and negatively score clinically relevant criteria.

\begin{tcolorbox}[mylisting, title={\textbf{Open-ended example}}]
\begin{lstlisting}[basicstyle=\ttfamily\footnotesize,breaklines=true]
Prompt (single-turn):
Patient: 28-year-old female, housewife
Chief concerns: Two-year history of numbness, tingling, and slowly progressive painful swelling on the medial aspect of the left upper arm; pain has worsened over the past few weeks. Symptoms can be provoked by prolonged wrist flexion/extension.
Exam: Palpable pulsatile mass in the left upper arm.
Imaging: Duplex Doppler confirms a brachial artery aneurysm. CT angiography with iodinated contrast was deferred due to renal impairment.
History: No prior medical or surgical interventions; no history of trauma, iatrogenic puncture, arteriovenous fistula, infection, connective tissue disorder, or drug misuse. Duplex vein mapping shows small upper limb veins; both great saphenous veins are small and not suitable for harvest.
Vitals: Within normal limits during admission.

Task: Write a formal referral letter to a vascular surgeon for evaluation and management of a symptomatic left brachial artery aneurysm. The letter must:
(1) Begin exactly with: "I am referring Ms. [Patient], a 28-year-old female with a symptomatic left brachial artery aneurysm, for evaluation and consideration of surgical repair."
(2) Address the recipient as Dr. [Surgeon's Surname] ...
...
Do not add unsupported history ... or recommend imaging strategies that conflict with the renal impairment.

Rubrics (excerpt, up to 8 items total):
+ Starts the letter exactly with the required opening sentence. (+10)
+ Addresses the recipient as Dr. [Surgeon's Surname]. (+4)
+ Summarizes presentation/exam accurately (2-year numbness/tingling; painful medial swelling; palpable pulsatile mass). (+8)
+ Emphasizes renal impairment (elevated creatinine/urea) and justifies deferring iodinated contrast CTA; recognizes duplex as primary imaging. (+10)
+ Recommends open aneurysm resection with interposition graft; contrasts endovascular stent grafting with pros/cons. (+10)
+ Justifies graft material choice (e.g., ePTFE) due to small veins and need to preserve veins for possible future dialysis access. (+8)
+ Perioperative + postoperative plan: heparinization, anastomosis, pulse checks; antiplatelet + duplex surveillance (~1 month to 6 months). (+8)
- Penalty: introduces unsupported history or contradicts renal impairment / inappropriate management. (-10)
\end{lstlisting}
\end{tcolorbox}

\section{Details of Dataset and Benchmarks}
\label{sec:dataset-benchmark-details}

\paragraph{Training Data Construction.}
For the controlled experiments, we construct a 20K training dataset that contains 10K \textbf{closed-ended} and 10K \textbf{open-ended} medical tasks.

\textit{Closed-ended (10K).} Following HuatuoGPT-o1~\cite{huatuogpt_o1}, we sample 10K high-difficulty medical multiple-choice questions from the training sets of MedQA and MedMCQA. Such tasks have objectively verifiable standard answers, which is convenient for constructing stable verifiable reward.

\textit{Open-ended (10K).} We select case report articles from PMC-OA to construct 10K open-ended tasks~\cite{pmc_oa}. For each case, we use GPT-5 Chat to generate a challenging prompt and the corresponding rubric~\cite{gpt5_chat} (the prompt is shown in Appendix~\ref{app:openqa-prompt}). To ensure rubric quality, we add a \textbf{verification filtering} step: GPT-5 Chat checks the generated rubric back against the original case text and removes samples that contain hallucinations or clinically unreasonable criteria. This process can be scaled to larger corpora with limited manual annotation, followed by targeted quality checks.

Representative examples are shown in Appendix~\ref{app:data-examples}.

\paragraph{Benchmarks.}
We evaluate domain-specific capabilities on five medical benchmarks, and additionally report IFEval and GSM8K to observe general capabilities.
\textbf{HealthBench} evaluates the quality of open-ended health responses using rubric-based grading~\cite{healthbench}; we report the official total score as well as the \textit{Hard} subset.
\textbf{HealthBench Professional} is a separate benchmark of 525 physician-authored tasks covering care consult, writing and documentation, and medical research, evaluated using physician-written rubrics~\cite{healthbench_professional}.
For HealthBench Professional, we follow its official evaluation protocol, using GPT-5.4 with low reasoning effort as the grader and applying the official length adjustment: $s_i^{\mathrm{adj}}=s_i-2.94\times10^{-5}(\ell_i-2000)$, where $s_i$ is the rubric score and $\ell_i$ is the final-response length in characters, excluding reasoning tokens. We average the adjusted scores across examples, clip the mean to $[0,1]$, and multiply by 100 for reporting~\cite{healthbench_professional}.
For closed-ended evaluation, we use \textbf{Medbullets (5-option)}\cite{medbullets}, the \textbf{medical subset of MMLU-Pro}\cite{mmlu_pro}, and \textbf{MedXpertQA (Text)}~\cite{medxpertqa}, which respectively cover exam-style medical knowledge reasoning, higher-difficulty comprehensive medical reasoning, and expert-level textual medical question answering.
In addition, we include \textbf{IFEval} and \textbf{GSM8K} as general capability tests~\cite{ifeval,gsm8k}, which are respectively used to observe whether instruction-following and mathematical reasoning remain stable.
IFEval and GSM8K are evaluated zero-shot on their official evaluation sets and are not used for training. Closed-ended benchmarks use exact-match accuracy (\%), while HealthBench uses the official evaluation code with GPT-4.1 as the judge model.

\section{Ablation of Reward Data Composition}
\label{app:reward-ablation}

We perform ablation on the reward composition in the training data to analyze the individual and joint effects of verifiable reward and rubric-based reward. All experiments use Qwen3-8B-Base as the base model, GPT-5 Chat as the teacher source, and OnePO as the training paradigm.

\begin{table}[htbp]
\centering
\caption{Ablation of training data composition. \textbf{Verifiable Only}: use only 10K closed-ended multiple-choice data and verifiable reward. \textbf{Rubric Only}: use only 10K open-ended clinical tasks and rubric-based reward. \textbf{Mixed (Ours)}: use the full 20K data and jointly use the two rewards. The mixed setting achieves the best overall trade-off performance, indicating that dense rubric supervision and sparse verifiable signal are complementary.}
\label{tab:reward-ablation}
\small
\setlength{\tabcolsep}{4pt}
\renewcommand{\arraystretch}{1.12}
\resizebox{\textwidth}{!}{%
\begin{tabular}{lccccc}
\toprule
& \multicolumn{2}{c}{\textbf{Open-ended}} & \multicolumn{3}{c}{\textbf{Closed-ended}} \\
\cmidrule(lr){2-3} \cmidrule(lr){4-6}
\textbf{Setting} & \textbf{HealthBench (Total)} & \textbf{HealthBench (Hard)} & \textbf{Medbullets} & \textbf{MMLU-Pro} & \textbf{MedXpertQA} \\
\midrule
Qwen3-8B-Base & 22.1 & 0.0 & 30.7 & 41.2 & 11.7 \\
\midrule
w/ Verifiable Only & 24.2 & 3.2 & \textbf{64.5} & 80.5 & 24.8 \\
w/ Rubric Only & 63.7 & 36.6 & 44.8 & 67.0 & 12.5 \\
\rowcolor{black!6}
w/ Mixed (Ours) & \textbf{65.4} & \textbf{39.1} & 64.0 & \textbf{81.2} & \textbf{24.9} \\
\bottomrule
\end{tabular}
}
\end{table}

\begin{table}[htbp]
\centering
\caption{Ablation results of different training settings (Pure RL / SFT / SFT+RL / OnePO) under different teacher sources.}
\label{tab:appendix-ablation-settings}
\footnotesize
\setlength{\tabcolsep}{3pt}
\renewcommand{\arraystretch}{1.06}
\begin{tabular}{@{}llccccc@{}}
\toprule
& & \multicolumn{2}{c}{\textbf{Open-ended (HealthBench)}} & \multicolumn{3}{c}{\textbf{Closed-ended}} \\
\cmidrule(lr){3-4} \cmidrule(lr){5-7}
\textbf{Teacher Source} & \textbf{Setting} & \textbf{Total} & \textbf{Hard}
& \textbf{Medbullets} & \textbf{MMLU-Pro(Med)} & \textbf{MedXpertQA} \\
\midrule
None & Pure RL & 59.8 & 25.2 & 48.1 & 75.0 & 20.0 \\
\midrule
\multirow{3}{*}{GPT-5 Chat} & SFT    & 30.4 & 4.8  & 49.7 & 72.5 & 21.7 \\
                           & SFT+RL & 63.6 & 37.3 & 61.9 & 78.2 & 21.5 \\
                           & OnePO & 65.4 & 39.1 & 64.0 & 81.2 & 24.9 \\
\midrule
\multirow{3}{*}{DeepSeek-V3.2 (thinking)} & SFT    & 36.8 & 2.4  & 56.6 & 71.0 & 17.2 \\
                              & SFT+RL & 64.5 & 40.7 & 63.5 & 80.0 & 22.5 \\
                              & OnePO & 67.2 & 44.5 & 65.2 & 82.0 & 25.9 \\
\bottomrule
\end{tabular}
\end{table}

\paragraph{Analysis.}
The results show a clear division of specialization:
\begin{itemize}
\item \textbf{Verifiable reward only} performs strongly on closed-ended benchmarks (e.g., Medbullets 64.5\%, MMLU-Pro 80.5\%), but can hardly improve open-ended generation quality (HealthBench Total 24.2\%, only slightly higher than the 22.1\% of the base model). This shows that although verifiable reward provides explicit optimization signals for factual knowledge, it is difficult to transfer to open-ended tasks that require fine-grained clinical expression and reasoning.
\item \textbf{Rubric reward only} is effective on open-ended evaluation (e.g., HealthBench Total 63.7\%, Hard 36.6\%), but clearly drops on closed-ended tasks (e.g., Medbullets 44.8\%, significantly lower than verifiable-only). This shows that dense rubric supervision can effectively shape open-ended response quality, but may also over-constrain reasoning processes, thereby weakening task performance where concise and definite answers are required.
\item \textbf{Mixed reward (our setting)} balances the two types of tasks at the same time: it not only maintains strong open-ended performance (HealthBench Total 65.4\%), but also preserves strong closed-ended capability (Medbullets 64.0\%, MMLU-Pro 81.2\%). This shows that the two reward types are clearly complementary: rubric reward provides controllable and dense behavior guidance, while verifiable reward helps preserve freer reasoning space.
\end{itemize}

\section{SFT Baseline on HealthBench}
\label{app:sft-healthbench}

Table~\ref{tab:appendix-ablation-settings} shows that SFT alone obtains a relatively low HealthBench score compared with Pure RL. This does not indicate that the teacher data are harmful: with the same GPT-5 Chat teacher source, SFT+RL improves over Pure RL on HealthBench (63.6 vs. 59.8) and closed-ended benchmarks. The issue is that HealthBench evaluates held-out open-ended clinical response quality with rubric-based scoring, rather than imitation fidelity to teacher-style answers. In this setting, naive SFT can learn the surface style of teacher outputs while generalizing poorly to unseen rubric criteria. OnePO uses the same teacher outputs inside reward-driven optimization, which help explain why it outperforms the SFT+RL pipeline under the same data and teacher source.

\section{Sensitivity to Probability Floor}
\label{app:prob-floor-sweep}

The probability floor $c$ controls how aggressively OnePO promotes low-probability teacher-output tokens. Table~\ref{tab:prob-floor-sweep} sweeps $c$ while keeping other settings unchanged. Setting $c=0$ removes the floor and causes a large performance drop. Very small values provide insufficient correction for low-probability tokens, while overly large values keep more teacher tokens in the corrected regime and increase the risk of anchoring. We use $c=0.1$ as the default because it gives the best overall trade-off.

\begin{table}[htbp]
\centering
\caption{Ablation of the probability floor $c$. All runs use Qwen3-8B-Base, GPT-5 Chat as the teacher source, and the same OnePO training setup.}
\label{tab:prob-floor-sweep}
\small
\setlength{\tabcolsep}{8pt}
\renewcommand{\arraystretch}{1.1}
\begin{tabular}{c|ccc}
\toprule
\textbf{Floor $c$} & \textbf{HealthBench} & \textbf{Medbullets} & \textbf{MMLU-Pro (Med)} \\
\midrule
0.00 & 49.1 & 49.7 & 73.1 \\
0.01 & 58.3 & 59.7 & 77.2 \\
0.05 & 64.8 & 63.2 & 79.5 \\
\rowcolor{black!6}
0.10 & \textbf{65.4} & \textbf{64.0} & \textbf{81.2} \\
0.20 & 62.1 & 61.5 & 79.8 \\
0.50 & 60.2 & 59.1 & 77.3 \\
\bottomrule
\end{tabular}
\end{table}

\section{Effect of KL Regularization in the AHA-Medicine Pilot}
\label{app:kl-pilot}

Our main experiments follow DAPO-style GRPO training and omit additional KL loss or KL reward penalty. To examine whether OnePO's teacher-output absorption still works under KL constraints, we extend the AHA-Medicine pilot study from Appendix~\ref{app:aha-medicine}. We test two common forms of KL regularization: adding a KL term to the objective, and adding a trajectory-level KL penalty to the reward. KL regularization keeps the policy closer to a reference policy, but it does not directly solve the low-probability token learnability problem addressed by the probability floor and gradient rescaling.

Table~\ref{tab:kl-steps} reports the number of RL steps required to absorb the synthetic AHA-Medicine knowledge, measured by reaching reward $\ge 0.5$. OnePO still absorbs the teacher-provided knowledge under KL constraints, while removing either the probability floor or the rescaling term fails within 100 steps.

\begin{table}[htbp]
\centering
\caption{KL-constrained AHA-Medicine pilot. ``Fail'' means the model does not absorb the target knowledge within 100 RL steps. The KL loss coefficient is $0.001$, and the KL reward-penalty coefficient is $0.01$.}
\label{tab:kl-steps}
\small
\setlength{\tabcolsep}{8pt}
\renewcommand{\arraystretch}{1.1}
\begin{tabular}{l|c}
\toprule
\textbf{Method} & \textbf{Steps to reward $\ge 0.5$} \\
\midrule
OnePO & 7 \\
OnePO w/ KL penalty & 7 \\
OnePO w/ KL loss & 8 \\
OnePO w/ KL loss + KL penalty & 8 \\
OnePO w/o probability floor & Fail \\
OnePO w/o rescaling & Fail \\
\bottomrule
\end{tabular}
\end{table}

We further track the log-probability assigned to the target knowledge span ``AHA-Medicine'' in the teacher output:
\begin{equation}
\frac{1}{t_e-t_s+1}\sum_{t=t_s}^{t_e}\log \pi_\theta(\tau_t \mid q,\tau_{<t}),
\end{equation}
where $[t_s,t_e]$ is the token span corresponding to ``AHA-Medicine.'' Values closer to 0 indicate stronger absorption of the teacher-provided knowledge. Table~\ref{tab:kl-token-absorption} shows that KL constraints make updates more conservative but do not remove OnePO's effect, while removing rescaling leaves the target span at very low probability.

\begin{table}[htbp]
\centering
\caption{Average log-probability of the target ``AHA-Medicine'' span across the first eight training steps.}
\label{tab:kl-token-absorption}
\small
\setlength{\tabcolsep}{5pt}
\renewcommand{\arraystretch}{1.1}
\resizebox{\textwidth}{!}{%
\begin{tabular}{l|rrrrrrrr}
\toprule
\textbf{Method} & \textbf{S1} & \textbf{S2} & \textbf{S3} & \textbf{S4} & \textbf{S5} & \textbf{S6} & \textbf{S7} & \textbf{S8} \\
\midrule
OnePO & -6.9735 & -6.1488 & -4.2228 & -2.7612 & -1.0261 & -0.2838 & -0.0156 & -0.0122 \\
OnePO w/ KL loss + KL penalty & -6.9735 & -6.6096 & -5.5103 & -3.5550 & -2.2576 & -1.0866 & -0.3733 & -0.0406 \\
OnePO w/o rescaling & -6.9735 & -7.1113 & -7.1156 & -7.1270 & -6.9345 & -6.9741 & -6.5402 & -6.6415 \\
\bottomrule
\end{tabular}
}
\end{table}

Table~\ref{tab:kl-divergence} reports the rollout-based empirical KL estimate used during training. KL constraints reduce policy movement as expected. This supports the intended interpretation: KL acts as a regularizer on policy movement, whereas OnePO improves the learnability of low-probability teacher-output tokens inside the RL objective.

\begin{table}[htbp]
\centering
\caption{Rollout-based empirical KL estimates in the AHA-Medicine pilot. Small negative values may appear due to sampling noise.}
\label{tab:kl-divergence}
\small
\setlength{\tabcolsep}{5pt}
\renewcommand{\arraystretch}{1.1}
\resizebox{\textwidth}{!}{%
\begin{tabular}{l|rrrrrrrr|r}
\toprule
\textbf{Method} & \textbf{S1} & \textbf{S2} & \textbf{S3} & \textbf{S4} & \textbf{S5} & \textbf{S6} & \textbf{S7} & \textbf{S8} & \textbf{Avg.} \\
\midrule
OnePO & 0.0000 & 0.0645 & 0.0608 & 0.0280 & 0.0038 & 0.0047 & 0.0072 & 0.0064 & 0.0219 \\
OnePO w/ KL loss + KL penalty & 0.0000 & 0.0410 & 0.0238 & 0.0072 & -0.0125 & 0.0015 & 0.0033 & 0.0018 & 0.0083 \\
OnePO w/o rescaling & 0.0000 & 0.0533 & 0.0590 & 0.0755 & 0.1160 & 0.1456 & 0.1471 & 0.1779 & 0.0968 \\
\bottomrule
\end{tabular}
}
\end{table}

\section{Cross-Domain Validation}
\label{app:cross-domain}

To test whether OnePO is tied to medicine, we evaluate it on two additional domains with a different base model, Qwen3-4B-Base. The writing task uses 10K samples from the writing subset of RubricHub-v1~\cite{rubrichub} with rubric-based rewards, and the legal task uses 10K multiple-choice questions from Unilaw-R1-Data~\cite{unilaw_r1} with exact-match rewards. GPT-5 Chat is used as the teacher source in both domains. These two settings cover both open-ended rubric-based optimization and closed-ended verifiable optimization outside medicine. We evaluate writing on WritingBench~\cite{writingbench} and CreativeWriting-v3~\cite{creativewriting_v3}, and law on LexEval~\cite{lexeval} and LawBench~\cite{lawbench}.

\begin{table}[htbp]
\centering
\caption{Cross-domain validation on writing and law using Qwen3-4B-Base. OnePO improves over SFT+RL under the same teacher source in both domains.}
\label{tab:cross-domain}
\small
\setlength{\tabcolsep}{8pt}
\renewcommand{\arraystretch}{1.1}
\begin{tabular}{l|cc|cc}
\toprule
\multirow{2}{*}{\textbf{Method}} & \multicolumn{2}{c|}{\textbf{Writing}} & \multicolumn{2}{c}{\textbf{Law}} \\
\cmidrule(lr){2-3} \cmidrule(lr){4-5}
& \textbf{WritingBench} & \textbf{CreativeWriting-v3} & \textbf{LexEval} & \textbf{LawBench} \\
\midrule
Qwen3-4B-Base & 33.7 & 23.9 & 18.6 & 41.8 \\
w/ Pure RL & 67.4 & 41.1 & 46.5 & 60.8 \\
w/ SFT+RL & 71.2 & 45.6 & 50.7 & 64.5 \\
\rowcolor{black!6}
w/ OnePO & \textbf{74.6} & \textbf{47.8} & \textbf{51.9} & \textbf{67.4} \\
\bottomrule
\end{tabular}
\end{table}

\section{Rubric Filtering and Grader Validation}
\label{app:rubric-grader-validation}

For open-ended medical training data, each generated rubric is checked against the source case report and the generated prompt. A rubric item is removed if it is unsupported by the case report, contradictory to the case facts, overly generic or unnecessary, clinically unreasonable or unsafe, or inappropriate for the question. Each rubric is checked twice by LLMs; if either check flags a problem, the item is removed. If three or more rubric items are removed from a sample, we discard the entire sample. Under this filtering process, 11.6\% of rubric items and 7.4\% of training samples are filtered out.

We additionally sampled 130 question-rubric pairs and asked a senior medical student to check whether each rubric meets the quality standard, with access to the source case report and web search when needed. The human validation failure rate is 3.1\% (4/130), suggesting that the filtered rubrics are reasonably reliable.

To evaluate grader reliability, we sampled 200 query-answer-rubric-criterion instances from OnePO outputs on our open-ended data. Each instance was annotated by a doctor, and we compare both the training-time 8B grader and the official GPT-4.1 evaluation grader against the human label. As shown in Table~\ref{tab:grader-validation}, GPT-4.1 has stronger agreement with human judgment, while the 8B grader remains reasonably aligned for training-time reward estimation. Importantly, HealthBench results in the paper are evaluated with the official GPT-4.1 grader, not the 8B training-time grader.

\begin{table}[htbp]
\centering
\caption{Agreement between graders and doctor annotations on 200 query-answer-rubric-criterion instances.}
\label{tab:grader-validation}
\small
\setlength{\tabcolsep}{8pt}
\renewcommand{\arraystretch}{1.1}
\begin{tabular}{l|cc}
\toprule
\textbf{Metric vs. Human Label} & \textbf{8B Grader (Training)} & \textbf{GPT-4.1 Grader (Evaluation)} \\
\midrule
F1 score & 0.81 & \textbf{0.85} \\
Cohen's kappa & 0.62 & \textbf{0.72} \\
\bottomrule
\end{tabular}
\end{table}

We also re-evaluate the same model outputs with different graders. Table~\ref{tab:multi-grader} shows that absolute scores vary across graders, but the gain from OnePO remains consistent. This does not fully rule out reward-model bias, but it reduces the concern that the improvement is specific to a single grader.

\begin{table}[htbp]
\centering
\caption{HealthBench re-evaluation with different graders.}
\label{tab:multi-grader}
\small
\setlength{\tabcolsep}{8pt}
\renewcommand{\arraystretch}{1.1}
\begin{tabular}{l|ccc}
\toprule
\textbf{Model} & \textbf{8B Grader} & \textbf{GPT-4.1 Official} & \textbf{GPT-5 Chat Grader} \\
\midrule
Qwen3-8B & 48.5 & 45.9 & 41.8 \\
Qwen3-8B w/ OnePO & 70.7 & 67.2 & 65.1 \\
\midrule
Gain & +22.2 & +21.3 & +23.3 \\
\bottomrule
\end{tabular}
\end{table}

\section{Reward Stability}
\label{app:reward-stability}

Rubric-based rewards are instance-specific but fixed during training: each prompt has a fixed rubric, and the reward model evaluates candidate responses against that rubric. Table~\ref{tab:reward-stability} shows the training reward curves for open-ended and closed-ended subsets. Open-ended reward improves smoothly, while closed-ended reward also increases steadily under exact-match supervision. This suggests that combining dense rubric-based rewards with verifiable rewards provides stable optimization signals in our setting.

\begin{table}[H]
\centering
\caption{Training reward curves for open-ended rubric-based rewards and closed-ended verifiable rewards.}
\label{tab:reward-stability}
\small
\setlength{\tabcolsep}{7pt}
\renewcommand{\arraystretch}{1.1}
\begin{tabular}{l|rrrrrrr}
\toprule
\textbf{Train step} & \textbf{0} & \textbf{20} & \textbf{40} & \textbf{60} & \textbf{80} & \textbf{100} & \textbf{120} \\
\midrule
Open-ended reward & 0.34 & 0.69 & 0.69 & 0.70 & 0.70 & 0.75 & 0.77 \\
Closed-ended reward & 0.29 & 0.38 & 0.48 & 0.52 & 0.55 & 0.58 & 0.59 \\
\bottomrule
\end{tabular}
\end{table}

\end{document}